%% file: template.tex
\documentclass{article}

\usepackage{arxiv}
\input{math_commands}

\usepackage[utf8]{inputenc} 
\usepackage[T1]{fontenc}    
\usepackage{hyperref}       
\usepackage{url}            
\usepackage{booktabs}       
\usepackage{amsfonts}       
\usepackage{nicefrac}       
\usepackage{microtype}      
\usepackage{lipsum}		
\usepackage{graphicx}
\usepackage{natbib}
\usepackage{doi}

\title{Time-Anchored Diffusion Language Models: Latent-Space Caching for Fast Generation}

\author{ {Joel Anto Paul} \\
	UT Austin \\
	\texttt{joelanto@utexas.edu} \\
	\And
    {Litu Rout} \\
		UT Austin \\
	\texttt{litu.rout@utexas.edu} \\
	\And
    {Aditya Akella} \\
		UT Austin \\
	\texttt{akella@cs.utexas.edu} \\
    \And
    {Sanjay Shakkottai} \\
		UT Austin \\
	\texttt{sanjay.shakkottai@utexas.edu} \\
}

\renewcommand{\headeright}{Technical Report}
\renewcommand{\undertitle}{Technical Report}
\renewcommand{\shorttitle}{Time-Anchored Diffusion Language Models}

\hypersetup{
pdftitle={A template for the arxiv style},
pdfsubject={q-bio.NC, q-bio.QM},
pdfauthor={David S.~Hippocampus, Elias D.~Striatum},
pdfkeywords={First keyword, Second keyword, More},
}

\usepackage[utf8]{inputenc} 
\usepackage[T1]{fontenc}    
\usepackage{hyperref}       
\usepackage{url}            
\usepackage{booktabs}       
\usepackage{longtable}      
\usepackage{amsfonts}       
\usepackage{nicefrac}       
\usepackage{microtype}      
\usepackage{xcolor}         
\usepackage{graphicx}

\usepackage{amsthm}

\usepackage{colortbl}
\usepackage{graphicx}
\usepackage{placeins}
\usepackage[ruled,linesnumbered]{algorithm2e}
\usepackage{wasysym}
\usepackage{mathtools}
\usepackage{subcaption}
\usepackage{algorithmic}
\usepackage{adjustbox}
\usepackage{amssymb}

\newtheorem{theorem}{Theorem}[section]

\newtheorem{remark}[theorem]{Remark}

\usepackage[most]{tcolorbox}
\usepackage{multirow}
\usepackage{tabularx}
\usepackage{enumitem}
\usepackage{mdframed}

\usepackage{wrapfig} 
\tcbuselibrary{skins, breakable}

\tcbset{
    promptbox/.style={
        enhanced,
        breakable,
        colback=red!4,
        colframe=red!55!black,
        boxrule=0.7pt,
        arc=2pt,
        left=5pt,right=5pt,top=4pt,bottom=4pt,
        fonttitle=\bfseries\small,
        coltitle=black
    },
    responsebox/.style={
        enhanced,
        breakable,
        colback=yellow!10,
        colframe=yellow!55!black,
        boxrule=0.7pt,
        arc=2pt,
        left=5pt,right=5pt,top=4pt,bottom=4pt,
        fonttitle=\bfseries\small,
        coltitle=black
    }
}

\hypersetup{colorlinks=true,                
    breaklinks=true,                
    urlcolor= orange,                
    linkcolor= orange,   
    bookmarksopen=false,
    filecolor=black,
    citecolor=blue,
    linkbordercolor=orange
}

\begin{document}

\maketitle

\begin{abstract}
Recent work on anchored diffusion language models improves denoising by shaping an intermediate latent space with supervised important-token targets. In this work, we introduce \emph{time-based (self-supervised) anchoring}, which learns and reuses latent anchors without requiring such targets. Our key observation is that anchors encode persistent properties of the clean sequence, such as its semantic intent, global structure, or intermediate plan. Although their hidden representations become stale as the token canvas evolves, their semantic content remains useful across nearby diffusion times. 
This is implemented through a two-stage architecture consisting of a relatively expensive anchor network that generates the latent cache state and a lightweight denoising network that intelligently combines the cached latent state with the current state at each reverse step using a fusion module. This gives anchoring a latent-space caching interpretation: the anchor network is evaluated periodically, while its cached representation is reused across multiple reverse steps. We instantiate this framework as TADM:Post-train, which time-anchorizes pretrained DLMs, and TADM:Pretraining, which learns time-based anchors during pretraining. Applied to DiffusionGemma-26B, TADM:Post-train improves throughput by approximately $49\%$ to $79\%$ on several math, code, and STEM benchmarks (GSM8K, AIME26, GPQA-Diamond, LiveCodeBench-v6, HumanEval, MMLU-Pro). TADM:Pretraining reduces Transformer-layer computation by up to $38\%$ relative to a standard single-stage DLM, achieves up to $73\%$ higher measured throughput than ADLM.
\end{abstract}
\section{Introduction}
\label{sec:introduction}
Diffusion Language Models (DLMs) iteratively refine an entire sequence,
enabling bidirectional attention and parallel token generation
\citep{d3pm,sedd,mdlm,md4}. This iterative formulation enables revision and
error correction during generation~\citep{remdm,Duo,schiff2025simple}, and
has recently been scaled to large language models such as
LLaDA~\citep{llada}, Dream~\citep{dream}, and
DiffusionGemma~\citep{diffusiongemma}, demonstrating that diffusion-based
generation can support competitive reasoning, coding, and instruction
following at scale. Block diffusion models provide another direction,
combining autoregressive generation across blocks with parallel denoising
within each block to improve generation flexibility and efficiency
\citep{bd3lm,diffusiongemma}. Despite these advances, DLMs still require repeated
model evaluations across denoising steps, making inference substantially
more expensive than autoregressive decoding.

\begin{figure}[t]
    \centering
    \includegraphics[
        width=\columnwidth,
        height=8cm,
        keepaspectratio,
        trim=0 3cm 0 3cm,
        clip
    ]{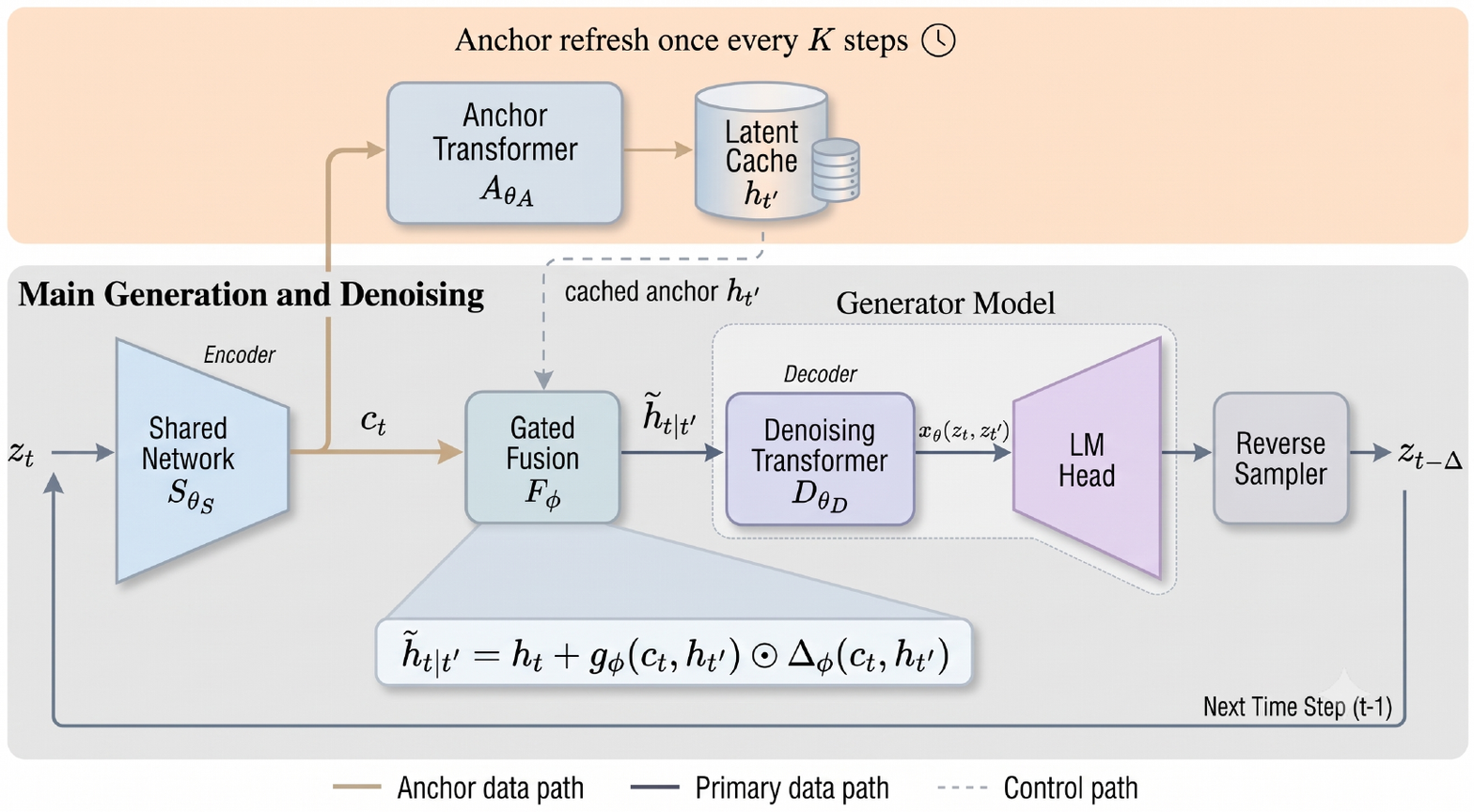}
    \caption{
        \textbf{Time-Anchored Diffusion Model (TADM).}
        The expensive anchor pathway is evaluated once every $K$ reverse
        steps to produce a latent anchor $\mathbf{h}_{t'}$, which is cached
        and reused. At each subsequent step, the lightweight shared network
        encodes the evolving state $\mathbf{z}_t$, and a gated fusion module
        combines its current representation $\mathbf{c}_t$ with the cached
        anchor before denoising. TADM therefore replaces repeated evaluation
        of the expensive anchor network with learned latent-space reuse.
    }
    \label{fig:tadm-overview}
    \vspace{-3ex}
\end{figure}

In recent work, the Anchored
Diffusion Language Model (ADLM) showed that denoising can be improved by first predicting
informative anchor tokens that reduce the conditional uncertainty of the
remaining sequence, and thus guide the denoiser~\citep{adlm}.
In this work, we introduce \emph{time-based (self-supervised) anchoring}, which learns and reuses latent anchors without requiring such targets. The key observation is that anchors describe persistent properties of the clean sequence, such as its semantic intent, global structure, or intermediate plan. Their semantic content should therefore remain useful across nearby diffusion times, even though their hidden representations become stale as the token canvas evolves.

We exploit this persistence by decomposing the model into four components: a lightweight shared network, an expensive anchor network, a fusion module, and a lightweight denoiser. The anchor representation changes slowly across nearby diffusion times and is therefore computed only at anchor-refresh steps and reused thereafter. At each reverse step, the shared network encodes the current canvas, and the fusion module combines this representation with the cached anchor before denoising. The anchor refresh interval therefore provides a tunable quality--compute trade-off by controlling how long the expensive anchor computation is reused (Figure~\ref{fig:tadm-overview}).

This construction also turns anchoring into a form of \emph{latent-space caching}. Existing methods reuse key-value states because consecutive diffusion states are similar \citep{dkvcache,d2cache,elasticcache}, but these local, layer-specific representations are not trained for temporal reuse. We instead train the model to produce a 
latent representation that can be reused for several denoising steps and corrected using the current state. This shifts the problem from detecting unchanged key-value states to learning the state that can be reused.

\noindent \textbf{Results.} Depending on whether a DLM is adapted after pretraining or trained from scratch, we instantiate time-based anchoring in two forms: \textbf{TADM:Post-train} and \textbf{TADM:Pretraining}. TADM:Post-train enables existing pretrained DLMs to acquire latent-cache reuse without repeating pretraining; on DiffusionGemma-26B, by training only a 3.1M-parameter fusion module while keeping the pretrained backbone frozen, it achieves $49\%$--$79\%$ higher throughput while preserving task accuracy of the corresponding baseline across the evaluated math, code, and STEM benchmarks. TADM:Pretraining instead learns cacheable latent representations directly during pretraining, reducing Transformer-layer computation by up to $38\%$ and achieving up to $73\%$ higher measured throughput than ADLM. At $T=2048$, it attains a MAUVE score of $0.650$ compared with $0.610$ for ReMDM while using $25\%$ fewer Transformer-layer evaluations than MDLM. 

\section{Preliminaries}
Let $\gV$ denote the vocabulary space with $V$ unique discrete tokens. Let $x = (x^1, x^2, \cdots, x^L) \text{ where } x^l \in \gV$, $l\in[L]$ be the sample sequence. We denote each element of the sample sequence as a one-hot vector; thus the input sequence can be denoted by $\rvx = (\rvx^1, \rvx^2, \cdots, \rvx^L)$ where $\rvx^l$ is a $V$-dimensional
one-hot vector with $\rvx^l[i] = 1$ when $x^l = i$, otherwise 0. We assume that the input sequence is sampled from an unknown distribution $q(\cdot)$ supported on $\gV^L$. We denote the Hadamard product of 
two vectors $\rva$ and $\rvb$ as $\rva \odot \rvb$ and dot product as $\langle \rva, \rvb \rangle$.

\subsection{Discrete Diffusion Models}
\label{sec:discrete-diffusion}

Discrete Diffusion Language Models (DLMs)~\citep{sohl2015deep,d3pm} define a forward corruption process that gradually transforms a clean sequence $\rvx$ into noise. Let $T$ denote the number of discrete diffusion steps, with $t(i)=i/T$ and $s(i)=(i-1)/T$. For simplicity, we write $t=t(i)$ and $s=s(i)$. The forward process independently corrupts each token according to
\vspace{-1ex}
\begin{align}
\label{eq:fwd}
    q(\rvz_t|\rvx) = \prod_{l=1}^{L} q(\rvz_t^l|\rvx), \quad   q(\rvz_t^l|\rvx) = \cat\left(\rvz_t^l; \alpha_t \rvx^l + (1-\alpha_t) \boldsymbol{\pi} \right),\quad l \in \{1,2,\cdots, L\},
\vspace{-1ex}
\end{align}
where $\alpha_t\in[0,1]$ is a monotonically decreasing noise schedule satisfying $\alpha_0=1$ and $\alpha_1=0$, and $\boldsymbol{\pi}$ denotes the limiting noise distribution. Defining $\alpha_{t|s}=\alpha_t/\alpha_s$, the corresponding one-step forward transition is
$q(\rvz_t^l|\rvz_s^l)=
\cat(\rvz_t^l;\alpha_{t|s}\rvz_s^l+
(1-\alpha_{t|s})\boldsymbol{\pi})$.

Because the clean token $\rvx^l$ is known during training, the posterior distribution of the preceding diffusion state has a closed form:
\vspace{-1ex}
\begin{equation}
\label{eq:rev-posterior}
q(\rvz_s^l\mid\rvz_t^l, \rvx^l)
=
\operatorname{Cat}\!\left(
\rvz_s^l;
\frac{
\left[
    \alpha_{t \mid s}\rvz_t^l
    +
    \left(1-\alpha_{t \mid s}\right)
    \mathbf{1}\boldsymbol{\pi}^{\top}\rvz_t^l
\right]
\odot
\left[
    \alpha_s\rvx^l
    +
    \left(1-\alpha_s\right)\boldsymbol{\pi}
\right]
}{
    \alpha_t\rvz_t^{l^{\top}}\rvx^l
    +
    \left(1-\alpha_t\right)
    \rvz_t^{l^{\top}}\boldsymbol{\pi}
}
\right).
\vspace{-1ex}
\end{equation}

At inference time, however, the clean sequence $\rvx$ is unknown. A neural denoiser therefore predicts a distribution over the clean token,
$\rvx_\theta^l(\rvz_t)$, and substitutes this prediction into the analytic posterior:
\vspace{-1ex}
\begin{align}
p_\theta(\rvz_s^l|\rvz_t)
=
q\!\left(
\rvz_s^l\mid
\rvz_t^l,
\rvx_\theta^l(\rvz_t)
\right).
\end{align}
The complete learned reverse process consequently factorizes as $p_\theta(\rvx,\rvz_{0:1})=p_\theta(\rvz_1)p_\theta(\rvx|\rvz_0)\prod_{i=1}^{T}p_\theta(\rvz_{s(i)}|\rvz_{t(i)})$.
The denoising model is trained by minimizing the discrete-time negative evidence lower bound (NELBO):
\vspace{-1ex}
\begin{align}
\gL_{\mathrm{NELBO}}(\rvx,\theta)
=&\;
\E_q\left[-\log p_\theta(\rvx\mid \rvz_{t(0)})\right]
\nonumber +
\sum_{i=1}^{T}
\E_q\left[
D_{\mathrm{KL}}\left(
q(\rvz_{s(i)}\mid \rvz_{t(i)},\rvx)
\,\Vert\,
p_\theta(\rvz_{s(i)}\mid \rvz_{t(i)})
\right)
\right]
\nonumber\\
&+
D_{\mathrm{KL}}\left(
q(\rvz_{t(T)}\mid\rvx)
\,\Vert\,
p_\theta(\rvz_{t(T)})
\right).
\label{eq:discrete-nelbo}
\vspace{-1ex}
\end{align}

\subsection{Anchored Diffusion Language Models}
\label{sec:ADLM}

A key limitation of standard masked diffusion models is that the corruption and reverse processes do not explicitly prioritize tokens according to their informativeness. Consequently, important tokens may remain masked until late in the reverse process, leaving the denoiser with insufficient context to accurately predict the remaining tokens. In particular, missing informative tokens can induce high conditional uncertainty over the rest of the sequence, making the denoising problem harder(See Appendix \ref{sec:ADLM-app} for more details). To address this limitation, \citet{adlm} introduced the \textit{Anchored Diffusion Language Model} (ADLM), which explicitly encourages the model to reason through a small set of informative tokens, referred to as \textit{anchor tokens}.

\textbf{Two-stage Parameterization.}
ADLM decomposes the standard denoising network into two stages: an \textit{anchor network} and a \textit{denoiser network}. Given a noisy sequence $\rvz_t$, the anchor network first predicts a probability mixture $\rvy_{\theta_A}(\rvz_t)$ over important tokens. The denoiser network then predicts the clean-token distribution conditioned on these anchored predictions. Thus, the output prediction is parameterized as
\vspace{-1ex}
\begin{equation}
    \rvx_\theta(\rvz_t)
    =
    \rvx_{\theta_D}\bigl(\rvy_{\theta_A}(\rvz_t)\bigr),
    \vspace{-1ex}
\end{equation}
where $\theta=[\theta_A,\theta_D]$ denotes the parameters of the anchor and denoiser networks, respectively.

\textbf{Training Objective.}
ADLM jointly trains the anchor and denoiser networks using an \textit{Anchored Negative Evidence Lower Bound} (ANELBO) (\ref{eq:anelbo-old}). The objective is split as 
\vspace{-1ex} 
\begin{align}
\label{eq:anelbo}
&\gL_{\mathrm{ANELBO}}(\rvx; \theta_A, \theta_D) 
\coloneq 
\gL_{diffusion}(\rvx, \theta) + \gamma\gL_{Anchor}(\rvx, \theta)
\vspace{-1ex} 
\end{align}
where $\gamma$ controls the strength of anchor supervision. The first term trains the denoiser to reconstruct the clean sequence conditioned on the anchor-network output, while the auxiliary anchor loss directly encourages the anchor network to predict informative tokens.

\section{Time-Anchored Diffusion Framework}
\label{sec:stale-anchor}
ADLM demonstrates that anchoring on important tokens can simplify denoising by reducing the uncertainty of the remaining sequence. However, its training requires a task-dependent specification of which tokens should serve as anchors, and such supervision may not always be available. We introduce \emph{time-based anchoring}, where anchors are instead learned as latent representations that captures information persistent across diffusion time, such as the global structure or underlying parent variables governing the sequence. These latent anchors are learned directly from the denoising objective and reused across multiple diffusion steps, without requiring explicit anchor targets. This naturally enables \emph{latent-space caching}, where an informative representation computed at one diffusion state can be retained and reused as the state evolves.

We realize time-based anchoring by decomposing the model into four components: a \emph{shared embedding network}, an expensive \emph{anchor network}, a lightweight \emph{denoiser network}, and a learned \emph{fusion module}. The shared network first maps the current diffusion state into a common representation space. The anchor network processes this representation to produce an expensive latent anchor, which is computed once, cached, and reused across multiple reverse-diffusion steps. At each subsequent step, the shared network encodes the current diffusion state, and the fusion module combines this current representation with the cached anchor before passing the result to the denoiser. Training under this reuse pattern teaches the fusion and denoising components to operate with temporally stale anchor representations, enabling the expensive anchor computation to be cached across diffusion time.

\begin{remark}
A motivation behind anchoring is that not all tokens provide equal information during the reconstruction of a sequence. ADLM describes anchor tokens as a set of tokens, denoted by $X_A$, conditioned on which the uncertainty of estimating the remaining tokens $X_{\bar A}$ is minimized. Given an anchor budget d, these anchor tokens can therefore be viewed as: $A^\star \in \operatorname*{arg\,min}_{A\subseteq[L],\,|A|\leq d} H\!\left(X_{\bar A}\mid X_A\right).$ 
Hence, revealing or accurately predicting $X_A$ makes the remaining variables easier to estimate.

In an idealized two-tier graphical model motivating anchoring (see \cite{adlm}), a small set of invariant upper-tier parent variables renders the lower-tier token variables conditionally independent, so that $p(x_{\bar A}\mid x_A)=\prod_{i\in\bar A}p(x_i\mid x_A)$ \citep{koller2009probabilistic}. Diffusion changes the observed corruption state but not these underlying parent variables; consequently, their semantic content, which are these parent variables, is invariant across diffusion time, although the learned representation may become stale and require correction, which is where the time-based anchoring comes into the picture. Our time-anchored framework exploits this invariance by training the denoiser with stale anchors and refining them using the current state. This encourages the anchor representation to capture information that remains useful across diffusion time, corresponding in the idealized graphical model to the persistent upper-tier parent variables.
\end{remark}

We next introduce the framework for time-based anchoring.
\subsection{Model Decomposition}
We decompose the denoising network into four components: a lightweight \emph{shared network} $S_{\theta_S}$, an expensive \emph{anchor network} $A_{\theta_A}$, a tiny \emph{fusion module} $\Phi_\phi$ and a lightweight \emph{denoiser network} $D_{\theta_D}$. For a noisy sequence $\rvz_t$, a standard fresh forward pass can be written as
\vspace{-1ex} 
\begin{equation}
\mathbf{c}_t = S_{\theta_S}(\rvz_t)
\qquad
\mathbf{h}_t = A_{\theta_A}(\mathbf{c}_t)
\qquad
\rvx_\theta(\rvz_t,t) = D_{\theta_D}(\mathbf{h}_t).
\vspace{-1ex}
\end{equation}
where $\mathbf{c}_t$ denotes the current-state representation and
$\mathbf{h}_t$ denotes the deep contextual representation produced by
the anchor network.
The decomposition is chosen such that the anchor network contains a large fraction of the total model computation and is evaluated only periodically. The shared network and denoiser remain inexpensive enough to be evaluated at every denoising step. Importantly, the precise decomposition is architecture-dependent: $S_{\theta_S}$ may consist only of an embedding/current-state pathway or may additionally contain several Transformer layers as in case of TADM:Post-train.

At an anchor-refresh step $t'$, we compute and cache the anchor representation
$\mathbf{h}_{t'}=A_{\theta_A}(S_{\theta_S}(\rvz_{t'}))$. At a later reverse
step $t<t'$, before the next refresh, the anchor network is skipped; instead,
the current state is encoded, fused with the cached anchor, and passed to the
denoiser:
\vspace{-1ex}
\begin{equation}
\mathbf{c}_t = S_{\theta_S}(\rvz_t)
\quad
\mathbf{h}_{t'} = A_{\theta_A}(S_{\theta_S}(\rvz_{t'}))
\quad
\widetilde{\mathbf{h}}_{t\mid t'}
= \Phi_{\phi}(\mathbf{c}_t,\mathbf{h}_{t'})
\quad
\rvx_\theta(\rvz_t,\rvz_{t'})
= D_{\theta_D}(\widetilde{\mathbf{h}}_{t\mid t'})
\vspace{-1ex}
\end{equation}
Thus, at intermediate reverse steps, only the shared network, fusion module,
and denoiser are evaluated, while the expensive cached anchor
$\mathbf{h}_{t'}$ is reused until the next refresh.

\subsection{Fusion network}
Replacing the input to the denoiser with the stale representation $\mathbf{h}_{t'}$ is not sufficient as this does not have information about the current state $\mathbf{c}_t$. Therefore, we introduce a tiny Fusion module $\Phi_\phi$ which combines $\mathbf{c}_t$ with $\mathbf{h}_{t'}$. This module is very small in comparison to the Shared network, Anchor Network and the Denoiser network.
The fusion module computes the corrected anchor representation as follows:
\vspace{-1ex} 
\begin{equation}
    \widetilde{\mathbf{h}}_{t\mid t'}
    =
    \mathbf{h}_{t'}
    +
    \mathbf{g}_{\phi}(\rvc_t, \rvh_{t'})
    \odot
    \boldsymbol{\Delta}_{\phi}(\rvc_t, \rvh_{t'}).
\label{eq:stale-anchor-fusion}
\vspace{-1ex} 
\end{equation}

Viewing the complete time-anchored model as a function of the current
state $\rvz_t$ and the stale anchor state $\rvz_{t'}$,
we denote its clean-sequence prediction by
$\rvx_\theta(\rvz_t,\rvz_{t'})$. Since
$\mathbf h_{t'}=A_{\theta_A}(S_{\theta_S}(\rvz_{t'}))$, the prediction is
\vspace{-1ex} 
\begin{equation}
    \label{eq:DenoiserOp}
    \rvx_\theta
    \left(
        \rvz_t,\mathbf{z}_{t'}
    \right)
    =
    D_{\theta_D}
    \left(
        \widetilde{\mathbf{h}}_{t\mid t'}
    \right).
    \vspace{-1ex} 
\end{equation}

\subsection{Anchor-Cache inference}
During inference, the anchor cache is refreshed once in every $K$ steps using the anchor network, whereas the shared and denoiser networks are evaluated at each step. The inference policy is summarized in Algorithm~\ref{alg:anchor-cache-inference}.

If $L_S$, $L_A$, and $L_D$ denote the numbers of shared, anchor, and denoiser Transformer layers, respectively, then the full and cached inference costs are
\vspace{-1ex}
\begin{equation}
C_{\mathrm{full}}=T(L_S+L_A+L_D),
\qquad
C_{\mathrm{cache}}=T(L_S+L_D)+\left\lceil\frac{T}{K}\right\rceil L_A .
\label{eq:anchor-cache-compute}
\vspace{-1ex}
\end{equation}
Thus, ignoring the ceiling term, the relative Transformer-layer cost is
\[
\frac{C_{\mathrm{cache}}}{C_{\mathrm{full}}}
\approx
\frac{L_S+L_D+L_A/K}{L_S+L_A+L_D},
\]
showing that increasing the anchor refresh interval $K$ reduces computation by amortizing the expensive anchor network across multiple reverse steps.

\subsection{Cache-induced distribution}

Let $\pi_K$ denote the anchor-refresh policy and let
$d_{\theta,\pi_K}$ denote the joint distribution over current states, and
cached anchor representations induced by executing the
target model under this policy. Ideally, training for cached inference
would minimize
\vspace{-1ex} 
\begin{equation}
    \gL_{\mathrm{cache}}
    =
    \E_{(t,t')\sim \pi_K}
    \E_{
        (\rvz_t,\mathbf{h}_{t'})
        \sim d_{\theta,\pi_K}(\cdot\mid t,t')
    }
    \left[
        \ell
        \left(
            \rvx_\theta
            \left(
                \rvz_t,\rvz_{t'}
            \right),
            \rvx
        \right)
    \right].
\label{eq:cache-objective}
\vspace{-1ex}
\end{equation}
Sampling from $d_{\theta,\pi_K}$ would require the reverse generation process during training, which introduces additional computation. In the following sections, we will introduce two ways of minimizing Eq. \ref{eq:cache-objective} depending on whether the model is pretrained from scratch or adapted through post-training.

\section{ TADM:Post-train}
TADM:Post-train is a post-training method using DiffusionGemma~\citep{diffusiongemma} as the base model, which is a USDM. It uses the framework developed in Sec \ref{sec:stale-anchor} for pre-trained DLMs. It approximates $d_{\theta,\pi_K}$ by generating the next denoising states via the rollout of the model during training. Simply sampling the noisy states $\rvz_t$ and $\rvz_{t'}$ using the forward processes in Eq. \ref{eq:fwd} and Eq. \ref{eq:fwd_t_t'}, then computing $\rvh_{t'}=A_{\theta_A}\left(S_{\theta_S}(\rvz_{t'})\right)$ would introduce a mismatch as the distribution $d_{\theta,\pi_K}$ encountered during inference would be different from the distribution induced by Eq. \ref{eq:fwd} because we are using a pretrained model. Therefore, to minimize this mismatch to an extent, we do the following:

\textbf{Rollout Construction.} For a clean response canvas $\rvx_0$, we sample a rollout age $k\in K_{train}=\{0,1,2\}$ and a noisy state $\rvz_{t'}$ using Eq. \ref{eq:fwd}, where $t'\sim\mathcal{U}(0,1)$ and we use uniform noise. We then project $t'$ onto the $T$-step sampling grid. $i_{\mathrm{a}}=\operatorname{clip}\left(\left\lfloor t'T+\tfrac{1}{2}\right\rfloor, k+1,T \right)$. Let $i_{\mathrm{s}}=i_{\mathrm{a}}-k.$ The anchor is evaluated once at this state, and then cached. $\rvh_{t'}=A_{\theta_A}\left(S_{\theta_S}(\rvz_{t'})\right)$. Starting from $\widetilde{\rvz}_{i_a}=\rvz_{t'}$, the current model
performs $k$ reverse transitions using the same denoising and sampling
operations as Algorithm~\ref{alg:anchor-cache-inference}, while keeping
the cached anchor $\rvh_{t'}$ fixed throughout the rollout. Note that the sampler is of your choice and in this specific case we use the sampler used by Diffusion Gemma.
\begin{equation}
    \widetilde{\rvz}_{j-1}
    \sim
    \mathcal{S}\!\left(
        p_\theta\!\left(
            \cdot \mid
            \widetilde{\rvz}_{j},
            \rvh_{t'}
        \right),j
    \right),
    \qquad
    j=i_{\mathrm{a}},\ldots,i_{\mathrm{s}}+1.
    \label{eq:anchor-sft-rollout}
\end{equation}
Therefore, in TADM:Post-train, the end point $\widetilde{\rvz}_{i_{s}}$ is sampled from the student rather than synthesized from the clean target, and is equivalent to $\rvz_t$.

\textbf{Training Objective.}
The discrete rollout in Eq~\ref{eq:anchor-sft-rollout} is
stop-gradient. After sampling its endpoint, we recompute the prompt,
cached-anchor, and terminal-state paths with automatic differentiation.
This preserves the on-policy terminal state without backpropagating
through discrete token sampling and allows supervision to reach the
trainable parameters in the shared, anchor, fusion, and denoising
pathways.

For each training example, we sample an anchor time $t'$ and a cache
age $k\in\{0,1,2\}$. The corresponding current diffusion time is
\(
t=t'-\frac{k}{T},
\)
so that the cached anchor is computed $k$ reverse steps before the
current state. TADM:Post-train minimizes
\vspace{-0.5ex}
\begin{equation}
    \mathcal{L}_{\mathrm{cache}}
    =
    \E_{t',k}
    \E_{
        (\widetilde{\rvz}_{t},\mathbf{h}_{t'})
        \sim d_{\theta,\pi_K}(\cdot\mid t',k)
    }
    \left[
        \mathcal{L}_{\mathrm{CE}}
        +
        \lambda_{\mathrm{KD}}D_{\mathrm{KL}}\!\left(p_{\theta_0}^{\,l}(\cdot\mid\widetilde{\rvz}_{t})\,\Vert\,\widehat{\rvx}^{\,l}
    \right)
    \right],
    \qquad
    t=t'-\frac{k}{T},
    \label{eq:anchor-sft-objective}
\end{equation}
where $\lambda_{\mathrm{KD}}=1$ in our experiments, $\theta_0$ is the frozen pretrained model and $\widehat{\rvx}$ is the output prediction of the denoiser using $\widetilde{\rvz_t}$ as the input. 
The outer expectation
averages over the sampled anchor time and cache age, while the inner
expectation is over the model-generated current state and cached anchor
induced by the rollout policy.
The second term is the knowledge distillation loss to prevent unnecessary corruption of the output distribution.

\section{TADM:Pretraining}
Although rollout training provides a close approximation to the inference distribution, applying multiple model forwards during pretraining is quite expensive. Moreover, when training from scratch, the model generated trajectories would be poor. Therefore, we use analytically generated noisy states using the forward Eq. \ref{eq:fwd}. Our pretraining experiment is derived from the two-stage architecture described in Sec. \ref{sec:ADLM} parameterized by $\theta=[\theta_D, \theta_A, \theta_S, \phi]$. We first sample $\rvz_t$ and then further corrupt it to get $\rvz_{t'}$ which produces $\rvh_{t'}$.
\begin{align}
\label{eq:fwd_t_t'}
q(\rvz_{t'}^l|\rvz_t^l)
=
\cat\left(
\rvz_{t'}^l;
\frac{\alpha_{t'}}{\alpha_t}\rvz_t^l
+ \left(1-\frac{\alpha_{t'}}{\alpha_t}\right)\rvm
\right).
\end{align}
\textbf{TADM Parametrization.} The forward process follows the standard masked diffusion model formulation (\ref{eq:fwd}) with $\boldsymbol{\pi}=\rvm$ and we use the ReMDM's~\citep{remdm} reverse posterior as the reference distribution (\ref{eq:remdm-posterior}). To introduce anchor caching in the pretraining, we postulate a new parametrization. The cached reverse process factorizes as: 

\begin{tcolorbox}[colback=gray!10, colframe=gray!10, boxrule=0pt, arc=2pt, left=0pt, right=0pt, top=0pt, bottom=0pt, breakable]
\vspace{-2ex}
\begin{align}
    \label{eq:joint-inf-post-TADM}
    p_{\theta}(\rvx, \rvz_{0:1})
    &=
    p_{\theta}(\rvz_1)p_{\theta}(\rvx|\rvz_0)
    \prod_{i=1}^{T}
    p_{\theta}\!\left(
    \rvz_{s(i)}\mid \rvz_{t(i)},\rvz_{t'(i)}
    \right).
\end{align}
\end{tcolorbox}
where $t'(i)$ denotes the most recent refresh time of the anchor cache available to the denoising step $t(i)$.

The one-step reverse transition uses the fused prediction $\rvx_{\theta}(\rvz_t, \rvz_{t'})$ (\ref{eq:DenoiserOp}) in the remasking posterior used by ReMDM~\citep{remdm}
\begin{tcolorbox}[colback=gray!10, colframe=gray!10, boxrule=0pt, arc=2pt, left=0pt, right=0pt, top=0pt, bottom=0pt, breakable]
\vspace{-1ex}
\begin{align}
\label{eq:inf-post-TADM}
q(\rvz_s^l|\rvz_t^l, \rvx^l_{\theta}(\rvz_t, \rvz_{t'})) = 
\begin{cases}
\mathrm{Cat}(\rvz_s^l; (1-\sigma_t) \rvz_t^l + \sigma_t \rvm), & \rvz_t^l \neq \rvm, \\
\mathrm{Cat}(\rvz_s^l; \frac{\alpha_s - (1-\sigma_t)\alpha_t}{1-\alpha_t}\rvx^l_{\theta}(\rvz_t, \rvz_{t'}) + \frac{1-\alpha_s-\alpha_t\sigma_t}{1-\alpha_t}\rvm), &\rvz_t^l = \rvm,
\end{cases}
\end{align}
\end{tcolorbox}
where $\sigma_t$ denotes the remasking probability at time $t$.

\textbf{Training objective.} We optimize the objective defined in \textbf{Theorem \ref{thm:temporalnelbo}} with $\sigma_t=0$, which trains the current prediction using the stale anchor sampled at time $t'$.

\textbf{Anchor token selection}
We choose the anchor tokens as the clean-token targets at positions masked in $\rvz_t'$ so that the anchor network is able to produce good latent representations persistent across time.

\textbf{Time-Anchored Negative Evidence Lower Bound (T-ANELBO).}
Following ReMDM~\citep{remdm}, we define the variational path distribution
over the latent trajectory as
\begin{equation}
q_{\sigma}(\rvz_{0:1}\mid\rvx)
=
q_{\sigma}(\rvz_1\mid\rvx)
\prod_{i=1}^{T}
q_{\sigma}
\left(
\rvz_{s(i)}
\mid
\rvz_{t(i)},\rvx
\right),
\label{eq:remdm-joint-path}
\vspace{-0.5ex}
\end{equation}
where $q_{\sigma}(\rvz_{s(i)}\mid\rvz_{t(i)},\rvx)$ is the remasking
posterior in Eq.~\ref{eq:remdm-posterior}. Although this process is
non-Markovian because the transition is additionally conditioned on
$\rvx$, its marginals $q_{\sigma}(\rvz_t\mid\rvx)$ coincide with those
of the standard absorbing-state diffusion process~\citep{remdm}.
We use $q_{\sigma}(\rvz_{0:1}\mid\rvx)$ as the variational distribution
in the following T-ANELBO.

\begin{tcolorbox}[colback=gray!10, colframe=gray!10, boxrule=0pt, arc=2pt, left=0pt, right=0pt, top=0pt, bottom=0pt, breakable]
\begin{theorem}[T-ANELBO]
\label{thm:temporalnelbo}
Suppose the latent trajectory follows the variational path distribution
$q_{\sigma}(\rvz_{0:1}\mid\rvx)$ in
Eq.~\ref{eq:remdm-joint-path}, with learned reverse transition parameterized as
in Eq.~\ref{eq:inf-post-TADM}.
Denote by $\theta$ the collection of parameters for the networks.
Given a sequence $\rvx=(\rvx^l)_{l=1}^L$, let the anchor tokens be
denoted by $\rvy=(\rvy_l)_{l=1}^L$ obtained through an operator
$\gA(\cdot)$. Then, the time-anchored negative log-likelihood is bounded by:

\vspace{-1ex}
\begin{align*}
-\log p_\theta(x) + \gamma \gL_{\mathrm{Anchor}}(x;\theta) 
\;\leq\;
\gL_{\mathrm{T-ANELBO}}(x; \theta), \quad \text{where}
\vspace{-1ex}
\end{align*}
\begin{align*}
&\gL_{\mathrm{T-ANELBO}}(\rvx; \theta) 
\coloneq \E_{q_{\sigma}(\rvz_{0:1}|\rvx)}\left[-\log p_\theta(\rvx | \rvz_0) \right] + \\
&{%
\thinmuskip=1mu
\medmuskip=10.0mu
\thickmuskip=10.0mu
\nulldelimiterspace=0pt
\mathsurround=0pt
\sum_{i=1}^{T}\!
\E_{q_{\sigma}(\rvz_{t(i)}, \rvz_{t'(i)}|\rvx)}\!\Biggl[\sum_{l=1}^L \lambda_{t(i)}
\log\!  \langle \rvx^l_{\theta}(\rvz_{t(i)}, \rvz_{t'(i)}), \rvx^l\rangle
+ \gamma\lambda_{t'(i)} \log \langle \rvy^l_{A_{\theta_{A}}}(\rvz_{t'(i)}), \rvy^l \rangle
\Biggr]%
}
\end{align*}

with $\lambda_{t(i)} = \frac{(1-\sigma_{t(i)})\alpha_{t(i)} - \alpha_{s(i)}}{1-\alpha_{t(i)}}$ and $\gamma \geq 0$.
\end{theorem}

\end{tcolorbox}
The first cross entropy term trains the denoiser for explicit reuse of latent representations through time, whereas the second term provides optional anchor supervision when anchor tokens are available. Setting $\gamma=0$ yields fully self-supervised time-based anchoring.

\section{Experiments}
\subsection{TADM:Post-train}
\textbf{Setup.}
We finetune DiffusionGemma~\citep{diffusiongemma} using our framework on math, code and STEM sequences derived from UltraData-SFT-2605~\citep{ultradata-sft-2605} dataset and Nemotron-Post-Training-Dataset-v2~\citep{nemotronpostv2}. We trained our model on sequences without reasoning traces and therefore evaluate our model using the no-think mode. We train for 18750 steps using batch size 16. We freeze the base model during training. The Shared network is the embedding layer and the first two text layers. The next 20 layers act as the Anchor network and the final 8 layers, along with the LM-Head acts as the Denoiser network. We set $\gamma=0$. The fusion module uses attention to fuse the current state with the stale anchor state. It starts with 0 output so that the model begins from the frozen backbone. For evaluation, we used a canvas size of 256 and fixed-budget sampler with 48 denoising steps per canvas. For more details, refer to \ref{sec:app-tadm-post-train}.

\textbf{Benchmarks.}
We use the following few-shot settings and maximum generation lengths:
GSM8K~\citep{gsm8k} (5-shot, 1,024),
HumanEval~\citep{humaneval} (0-shot, pass@1, 1,024),
GPQA-Diamond~\citep{rein2024gpqa} (0-shot, 2,048),
MMLU-Pro~\citep{wang2024mmlu} (5-shot, 2,048),
LiveCodeBench-v6~\citep{lcb} (0-shot, pass@1, 2,000),
and AIME26~\citep{aime26} (0-shot, 4,096 tokens).

\textbf{Metrics.} We evaluate the model accuracy over the benchmarks, tokens per seconds generated by the model and the transformer call reduction averaged over 3 seeds. These are reported in Table \ref{tab:tadm-posttrain-k-ablation}.
\begin{table*}[t]
\centering
\caption{
\textbf{TADM:Post-train performance.}
Results are averaged over three seeds across math, code, and STEM
benchmarks and compared with the DiffusionGemma baseline under the same
fixed-step evaluation setting. Increasing the anchor refresh interval
enables progressively greater reuse of the cached latent representation,
yielding substantial inference speedups. TADM preserves essentially the
same task accuracy across benchmarks while achieving up to $1.79\times$
higher generation throughput. 
}
\label{tab:tadm-posttrain-k-ablation}

\resizebox{\textwidth}{!}{%
\begin{tabular}{lllrrrrr}
\toprule
Benchmark & Method & $K$
& Acc. (\%)
& $\Delta$ Acc. (pp)
& Tok./s
& Throughput
& Compute ($\times$ Base) \\
\midrule

GSM8K & DiffusionGemma & - & 94.79 & --    & 33.80 & 1.00$\times$ & 1.00$\times$ \\
GSM8K & TADM & 1 & 94.79 & $0.00$ & 35.72 & 1.06$\times$ & 1.00$\times$ \\
GSM8K & TADM & 2 & 94.79 & $0.00$ & 51.51 & 1.52$\times$ & 0.67$\times$ \\
GSM8K & TADM & 3 & 94.95 & $+0.16$ & 59.53 & 1.76$\times$ & 0.56$\times$ \\

\midrule

GPQA-D & DiffusionGemma & - & 67.00 & -- & 60.31 & 1.00$\times$ & 1.00$\times$ \\
GPQA-D & TADM & 1 & 67.00 & $0.00$ & 55.30 & 0.92$\times$ & 1.00$\times$ \\
GPQA-D & TADM & 2 & 66.16 & $-0.84$ & 79.89 & 1.32$\times$ & 0.67$\times$ \\
GPQA-D & TADM & 3 & 66.16 & $-0.84$ & 93.90 & 1.56$\times$ & 0.56$\times$ \\

\midrule

HumanEval & DiffusionGemma & - & 95.12 & -- & 41.84 & 1.00$\times$ & 1.00$\times$ \\
HumanEval & TADM & 1 & 95.12 & $0.00$ & 37.96 & 0.91$\times$ & 1.00$\times$ \\
HumanEval & TADM & 2 & 95.53 & $+0.41$ & 58.22 & 1.39$\times$ & 0.67$\times$ \\
HumanEval & TADM & 3 & 94.72 & $-0.40$ & 63.85 & 1.53$\times$ & 0.56$\times$ \\

\midrule

LCB-v6 & DiffusionGemma & - & 50.29 & -- & 58.35 & 1.00$\times$ & 1.00$\times$ \\
LCB-v6 & TADM & 1 & 50.29 & $0.00$ & 56.65 & 0.97$\times$ & 1.00$\times$ \\
LCB-v6 & TADM & 2 & 52.38 & $+2.09$ & 78.66 & 1.35$\times$ & 0.67$\times$ \\
LCB-v6 & TADM & 3 & 50.10 & $-0.19$ & 104.44 & 1.79$\times$ & 0.56$\times$ \\

\midrule

AIME26 & DiffusionGemma & - & 47.78 & -- & 62.79 & 1.00$\times$ & 1.00$\times$ \\
AIME26 & TADM & 1 & 47.78 & $0.00$ & 57.48 & 0.92$\times$ & 1.00$\times$ \\
AIME26 & TADM & 2 & 46.67 & $-1.11$ & 83.41 & 1.33$\times$ & 0.67$\times$ \\
AIME26 & TADM & 3 & 48.89 & $+1.11$ & 93.50 & 1.49$\times$ & 0.56$\times$ \\

\midrule

MMLU-Pro & DiffusionGemma & - & 77.15 & -- & 48.75 & 1.00$\times$ & 1.00$\times$ \\
MMLU-Pro & TADM & 1 & 77.15 & $+0.00$ & 45.45 & 0.93$\times$ & 1.00$\times$ \\
MMLU-Pro & TADM & 2 & 76.89 & $-0.26$ & 63.79 & 1.31$\times$ & 0.67$\times$ \\
MMLU-Pro & TADM & 3 & 76.58 & $-0.57$ & 77.33 & 1.59$\times$ & 0.56$\times$ \\
\bottomrule
\end{tabular}%
}
\vspace{-1.5ex}
\end{table*}

\textbf{Results.}
Table~\ref{tab:tadm-posttrain-k-ablation} shows that pretrained DLMs can be effectively \emph{time-anchorized} with TADM:Post-train, by training only on 3.1M parameters of the fusion module, substantially improving inference throughput while largely preserving task accuracy. At $K=3$, throughput improves by $1.49\times$--$1.79\times$ across the evaluated benchmarks. Importantly, this behavior also holds on harder reasoning and coding tasks such as AIME26, GPQA-Diamond, and LiveCodeBench-v6, where we obtain $1.49\times$, $1.56\times$, and $1.79\times$ throughput, respectively, with only minor changes in accuracy. These results suggest that reusing learned latent representations across nearby diffusion steps is an effective way to reduce repeated computation in pretrained DLMs.

\subsection{TADM:Pretraining}
\textbf{Setup}
We pre-train TADM on OpenWebText (OWT)~\citep{owt} for 1M optimization steps, chosen to match the training-token budget of the main comparison models. The 12-layer anchor and 6-layer denoiser use the Diffusion Transformer (DiT) architecture~\citep{dit}; they share the token embedding table, and the gated residual fusion module connects the anchor hidden states to the denoiser. The inference graph contains 224M unique parameters, including the fusion gate and excluding the training-only anchor LM head (38.85M parameters). We use the GPT-2 tokenizer and the ReMDM sampler~\citep{remdm}. We evaluate 1,024-token unconditional generations and report MAUVE, generative perplexity (Gen PPL), entropy, compute, and throughput. Higher is better except for Gen PPL and compute. GPT-2 Large is the evaluator for Gen PPL and MAUVE~\citep{mauve}.
The anchor refresh interval $K$ specifies the number of reverse steps between anchor updates: $K=1$ refreshes the anchor at every step. TADM:Pretraining$^\dagger$ represents $\gamma=0$, whereas TADM:Pretraining$^\ddagger$ uses the optional anchor supervision using $\gamma=3e-3$.
We report the inference parameter count, which excludes the 38.85M Anchor LM Head which is not used during inference. 
\begin{table*}[!h]
\vspace{-1ex}
\centering
\caption{
We compare TADM diffusion baselines across sampling budgets. For TADM, we use refresh intervals
$K=\{2,2,2,8,4,4\}$ for
$T=\{128,256,512,1024,2048,4096\}$, respectively. We report MAUVE, generative perplexity (Gen PPL), entropy, Transformer-layer evaluations normalized to MDLM over 5000 samples, and measured generation throughput over 20 generations of length 1,024. TADM periodically reuses its cached anchor representation, reducing Transformer computation while maintaining competitive generation quality.
}

\label{tab:main-results}

\resizebox{\textwidth}{!}{%
\begin{tabular}{lccccccccccccccc}
\toprule
Method & \multicolumn{3}{c}{MAUVE ($\uparrow$)} & \multicolumn{3}{c}{Gen PPL ($\downarrow$)} & \multicolumn{3}{c}{Entropy ($\uparrow$)} & \multicolumn{3}{c}{Compute ($\downarrow$)} & \multicolumn{3}{c}{Throughput ($\uparrow$)} \\
\midrule
Data & \multicolumn{3}{c}{1.00} & \multicolumn{3}{c}{14.8} & \multicolumn{3}{c}{5.44} & \multicolumn{3}{c}{-} & \multicolumn{3}{c}{-} \\
\midrule
AR \textit{(T=1024)} & \multicolumn{3}{c}{0.760} & \multicolumn{3}{c}{12.1} & \multicolumn{3}{c}{5.22} & \multicolumn{3}{c}{-} & \multicolumn{3}{c}{-} \\
\midrule
& \textit{T=2048} & \textit{T=4096} & & \textit{T=2048} & \textit{T=4096} & & \textit{T=2048} & \textit{T=4096} & & \textit{T=2048} & \textit{T=4096} & & \textit{T=2048} & \textit{T=4096} & \\
\midrule
SEDD (absorb, 170M) & 0.008 & 0.009 & & 103.2 & 102.5 & & 5.61 & 5.61 & & 24576 (1.00x) & 49152 (1.00x) & & 37.47 (1.54x) & 18.79 (1.34x) & \\
MDLM (170M) & 0.037 & 0.035 & & 51.3 & 50.9 & & 5.46 & 5.45 & & 24576 (1.00x) & 49152 (1.00x) & & 63.00 (2.59x) & 46.39 (3.31x) & \\
MDLM+FB (170M) & 0.197 & 0.243 & & 28.6 & 22.8 & & 5.28 & 5.18 & & 24576 (1.00x) & 49152 (1.00x) & & 37.01 (1.52x) & 25.28 (1.80x) & \\
MDLM+DFM (170M) & 0.294 & 0.269 & & 21.0 & 20.7 & & 5.19 & 5.17 & & 24576 (1.00x) & 49152 (1.00x) & & 29.98 (1.23x) & 15.48 (1.10x) & \\
ReMDM (170M)& 0.610 & 0.656 & & 22.8 & 17.6 & & 5.30 & 5.20 & & 24576 (1.00x) & 49152 (1.00x) & & 35.07 (1.44x) & 20.00 (1.43x) & \\
ADLM (293M) & 0.788 & 0.791 & & 20.3 & 15.9 & & 5.28 & 5.19 & & 36864 (1.50x) & 73728 (1.50x) & & 24.32 (1.00x) & 14.01 (1.00x) & \\
\rowcolor{orange!25}
TADM$^{\dagger}$ (224M) & 0.650 & 0.618 & & 21.77 & 17.22 & & 5.314 & 5.226 & & 18432 (0.75x) & 36864 (0.75x) & & 37.42 (1.54x) & 20.94 (1.49x) & \\
\rowcolor{orange!25}
TADM$^{\ddagger}$ (224M) & 0.646 & 0.637 & & 21.70 & 17.19 & & 5.293 & 5.221 & & 18432 (0.75x) & 36864 (0.75x) & & 37.35 (1.54x) & 20.94 (1.49x) & \\
\bottomrule
\end{tabular}
}

\vspace{1.0ex}
\resizebox{\textwidth}{!}{%
\begin{tabular}{lccccccccccccccc}
\toprule
Method & \multicolumn{3}{c}{MAUVE ($\uparrow$)} & \multicolumn{3}{c}{Gen PPL ($\downarrow$)} & \multicolumn{3}{c}{Entropy ($\uparrow$)} & \multicolumn{3}{c}{Compute ($\downarrow$)} & \multicolumn{3}{c}{Throughput ($\uparrow$)} \\
\midrule
& \textit{T=512} & \textit{T=1024} & & \textit{T=512} & \textit{T=1024} & & \textit{T=512} & \textit{T=1024} & & \textit{T=512} & \textit{T=1024} & & \textit{T=512} & \textit{T=1024} & \\
\midrule
SEDD (absorb, 170M) & 0.008 & 0.008 & & 107.2 & 104.7 & & 5.62 & 5.62 & & 6144 (1.00x) & 12288 (1.00x) & & 149.40 (1.83x) & 74.91 (1.73x) & \\
MDLM (170M) & 0.031 & 0.042 & & 53.0 & 51.3 & & 5.48 & 5.46 & & 6144 (1.00x) & 12288 (1.00x) & & 134.91 (1.66x) & 87.74 (2.03x) & \\
MDLM+FB (170M) & 0.100 & 0.133 & & 37.1 & 33.8 & & 5.38 & 5.35 & & 6144 (1.00x) & 12288 (1.00x) & & 119.26 (1.46x) & 62.26 (1.44x) & \\
MDLM+DFM (170M) & 0.211 & 0.254 & & 23.3 & 21.7 & & 5.23 & 5.20 & & 6144 (1.00x) & 12288 (1.00x) & & 119.02 (1.46x) & 59.73 (1.38x) & \\
ReMDM (170M)& 0.350 & 0.403 & & 21.1 & 28.6 & & 5.21 & 5.38 & & 6144 (1.00x) & 12288 (1.00x) & & 116.94 (1.44x) & 62.37 (1.44x) & \\
ADLM (293M) & 0.573 & 0.699 & & 31.6 & 25.4 & & 5.40 & 5.35 & & 9216 (1.50x) & 18432 (1.50x) & & 81.44 (1.00x) & 43.19 (1.00x) & \\
\rowcolor{orange!25}
TADM$^{\dagger}$ (224M) & 0.339 & 0.405 & & 21.38 & 28.65 & & 5.230 & 5.383 & & 6144 (1.00x) & 7680 (0.62x) & & 111.95 (1.37x) & 74.54 (1.73x) & \\
\rowcolor{orange!25}
TADM$^{\ddagger}$ (224M) & 0.363 & 0.441 & & 21.74 & 28.41 & & 5.233 & 5.344 & & 6144 (1.00x) & 7680 (0.62x) & & 111.50 (1.37x) & 74.34 (1.72x) & \\
\bottomrule
\end{tabular}
}

\vspace{1.0ex}
\resizebox{\textwidth}{!}{%
\begin{tabular}{lccccccccccccccc}
\toprule
Method & \multicolumn{3}{c}{MAUVE ($\uparrow$)} & \multicolumn{3}{c}{Gen PPL ($\downarrow$)} & \multicolumn{3}{c}{Entropy ($\uparrow$)} & \multicolumn{3}{c}{Compute ($\downarrow$)} & \multicolumn{3}{c}{Throughput ($\uparrow$)} \\
\midrule
& \textit{T=128} & \textit{T=256} & & \textit{T=128} & \textit{T=256} & & \textit{T=128} & \textit{T=256} & & \textit{T=128} & \textit{T=256} & & \textit{T=128} & \textit{T=256} & \\
\midrule
SEDD (absorb, 170M) & 0.007 & 0.007 & & 119.2 & 110.1 & & 5.65 & 5.63 & & 1536 (1.00x) & 3072 (1.00x) & & 588.31 (1.83x) & 299.94 (1.86x) & \\
MDLM (170M)& 0.015 & 0.023 & & 61.5 & 55.8 & & 5.52 & 5.49 & & 1536 (1.00x) & 3072 (1.00x) & & 470.07 (1.47x) & 242.39 (1.50x) & \\
MDLM+FB (170M)& 0.064 & 0.084 & & 42.8 & 39.6 & & 5.44 & 5.41 & & 1536 (1.00x) & 3072 (1.00x) & & 469.90 (1.47x) & 238.85 (1.48x) & \\
MDLM+DFM (170M)& 0.041 & 0.144 & & 37.9 & 26.5 & & 5.31 & 5.26 & & 1536 (1.00x) & 3072 (1.00x) & & 455.31 (1.42x) & 236.80 (1.47x) & \\
ReMDM (170M)& 0.057 & 0.216 & & 42.5 & 30.5 & & 5.43 & 5.34 & & 1536 (1.00x) & 3072 (1.00x) & & 461.76 (1.44x) & 234.84 (1.46x) & \\
ADLM (293M) & 0.140 & 0.349 & & 52.5 & 39.85 & & 5.52 & 5.46 & & 2304 (1.50x) & 4608 (1.50x) & & 320.62 (1.00x) & 161.31 (1.00x) & \\
\rowcolor{orange!25}
TADM:Pretraining$^{\dagger}$ (224M) & 0.084 & 0.239 & & 41.56 & 29.86 & & 5.433 & 5.339 & & 1536 (1.00x) & 3072 (1.00x) & & 445.23 (1.39x) & 222.28 (1.38x) & \\
\rowcolor{orange!25}
TADM:Pretraining$^{\ddagger}$ (224M) & 0.070 & 0.220 & & 43.06 & 30.59 & & 5.433 & 5.337 & & 1536 (1.00x) & 3072 (1.00x) & & 447.20 (1.39x) & 223.68 (1.39x) & \\
\bottomrule
\end{tabular}
}
\vspace{-1.5ex}
\end{table*}

\textbf{Baselines} We compare our method to ADLM~\citep{adlm}, MDLM~\citep{mdlm}, ReMDM~\citep{remdm}, SEDD~\citep{sedd}, and MDLM with DFM~\citep{gat2024discrete} and Forward-Backward (FB)~\citep{campbell2022continuous} sampling techniques. 

 \textbf{Inference Speed-up:} Table~\ref{tab:main-results} reports both the analytical Transformer-call reduction and measured total throughput for TADM:Pretraining. Total throughput counts all 1,024 generated positions per sample, including positions after the first end-of-text token. We normalize speed-up with respect to ADLM.

\textbf{Compute Cost:} We count Transformer-layer calls during inference:
$C=DT\cdot T+AT\cdot T/K$, where $DT$ and $AT$ are the numbers of denoiser and anchor Transformer layers. TADM:Pretraining, and ADLM use $DT=6$ and $AT=12$, while MDLM and SEDD use $DT=12$ and $AT=0$. We normalize compute to MDLM. This layer-call measure does not include the relatively smaller embedding, output-head, or fusion-gate operations.

\textbf{Results.}
Table~\ref{tab:main-results} shows that TADM achieves a favorable quality--compute trade-off. At larger sampling budgets, TADM matches or improves upon ReMDM in Gen PPL while remaining competitive with ADLM, but with substantially higher throughput; its MAUVE scores also remain competitive with ReMDM. Although TADM contains 224M unique parameters, time anchoring reduces the amount of Transformer computation executed at each reverse step. With 12 anchor and 6 denoiser layers, $K=2$ gives an effective cost of $6+12/2=12$ Transformer layers per step, matching the 12-layer ReMDM backbone, with only the lightweight fusion module as additional computation. Larger $K$ further reduces the number of Transformer-layer evaluations, yielding up to a $38\%$ compute reduction and $1.73\times$ the measured throughput of ADLM. These gains arise from reusing latent representations across diffusion steps rather than recomputing the full model at every step. A detailed refresh-interval sweep is provided in Table~\ref{tab:TADM-refresh-interval-sweep}.

\section{Conclusion}
\label{sec:conclusion}

We introduced time-based anchoring, which makes the latent space of a DLM cacheable by learning representations whose semantic content persists across diffusion time. A relatively more expensive anchor network is evaluated periodically, while a lightweight denoising network combines the cached latent state with the current state at every reverse step. The framework supports self-supervised anchors when $\gamma=0$ and combines self-discovered information with supervised important tokens when $\gamma>0$. We instantiate it through TADM:Post-train, which time-anchorizes pretrained DLMs, and TADM:Pretraining, which learns cacheable latent states during pretraining. TADM:Post-train improves DiffusionGemma throughput by approximately $49\%$ to $79\%$ across the evaluated benchmarks while largely preserving task accuracy. TADM:Pretraining reduces Transformer-layer computation by up to $38\%$ and achieves up to $73\%$ higher measured throughput than ADLM. Together, these results establish time-based anchoring as a general approach to caching latent computation and accelerating diffusion language models.

\textbf{Limitations.} TADM's generation quality degrades at larger anchor refresh intervals, limiting the duration of effective cache reuse. The accuracy of TADM:Post-train is also constrained by the capabilities of its frozen pretrained backbone.

\section*
{Acknowledgments}
This work was supported in part by NSF Grants 2112471 (NSF AI EDGE), 2505865 (NSF IFML), and 2326576 (NSF LDOS), and by the UT Austin InfraAI Center. We are grateful for computing support on the Vista GPU Cluster through the Center for Generative AI (CGAI) and the Texas Advanced Computing Center (TACC) at the University of Texas at Austin.

\bibliographystyle{unsrtnat}
\bibliography{references} 

\newpage
\appendix
\section{Additional Theory}
We provide the proof for the Time-Anchored NELBO Loss formulated as the training objective in TADM:Pretraining here.
\label{sec:addn-theory}
\begin{theorem}[T-ANELBO]
\label{thm:app-temporalnelbo}
Suppose the latent trajectory follows the variational path distribution
$q_{\sigma}(\rvz_{0:1}\mid\rvx)$ in
Eq.~\ref{eq:remdm-joint-path}, with learned reverse transition parameterized as
in Eq.~\ref{eq:inf-post-TADM}.
Denote by $\theta$ the collection of parameters for the networks.
Given a sequence $\rvx=(\rvx^l)_{l=1}^L$, let the anchor tokens be
denoted by $\rvy=(\rvy_l)_{l=1}^L$ obtained through an operator
$\gA(\cdot)$. Then, the time-anchored negative log-likelihood is bounded by:
\vspace{-1ex}
\begin{align*}
-\log p_\theta(x) + \gamma \gL_{\mathrm{Anchor}}(x;\theta) 
\;\leq\;
\gL_{\mathrm{T-ANELBO}}(x; \theta), \quad \text{where}
\vspace{-1ex}
\end{align*}
\begin{align*}
&\gL_{\mathrm{T-ANELBO}}(\rvx; \theta) 
\coloneq \E_{q_{\sigma}(\rvz_{0:1}|\rvx)}\left[-\log p_\theta(\rvx | \rvz_0) \right] + \\
&{%
\thinmuskip=1mu
\medmuskip=10.0mu
\thickmuskip=10.0mu
\nulldelimiterspace=0pt
\mathsurround=0pt
\sum_{i=1}^{T}\!
\E_{q_{\sigma}(\rvz_{t(i)}, \rvz_{t'(i)}|\rvx)}\!\Biggl[\sum_{l=1}^L \lambda_{t(i)}
\log\!  \langle \rvx^l_{\theta}(\rvz_{t(i)}, \rvz_{t'(i)}), \rvx^l\rangle
+ \gamma\lambda_{t'(i)} \log \langle \rvy^l_{A_{\theta_{A}}}(\rvz_{t'(i)}), \rvy^l \rangle
\Biggr]%
}
\end{align*}

with $\lambda_{t(i)} = \frac{(1-\sigma_{t(i)})\alpha_{t(i)} - \alpha_{s(i)}}{1-\alpha_{t(i)}}$ and $\gamma \geq 0$.
\end{theorem}

\begin{proof}
We show for L=1, the proof can be easily generalized to L tokens. Starting from the standard negative log-likelihood:
\vspace{-1ex}
\begin{align*}
    -\log p_\theta(\rvx) 
    &= -\log \int p_\theta(\rvx, z_{0:1}) \, d(z_{0:1}) \\
    &= -\log \int 
    \frac{p_\theta(\rvx, z_{0:1})}
    {q_{\sigma}(z_{0:1}|\rvx)}
    q_{\sigma}(z_{0:1}|\rvx) \, d(z_{0:1})
\end{align*}

Applying Jensen's inequality and using Eq.~\ref{eq:joint-inf-post-TADM}
to factorize the joint distribution, we get:
\vspace{-1ex}
\begin{align*}
    -\log p_\theta(\rvx) 
    &\leq 
    \E_{q_{\sigma}(z_{0:1}|\rvx)}\left[
        -\log p_\theta(\rvx|z_0)
        + \log \frac{q_{\sigma}(z_1|\rvx)}{p_\theta(z_1)}
        + \sum_{i=1}^T
        \log
        \frac{
            q_{\sigma}(z_{s(i)}|z_{t(i)},\rvx)
        }{
            p_\theta(z_{s(i)}|z_{t(i)},z_{t'(i)})
        }
    \right]
    \coloneq \gL_{\mathrm{NELBO}}(\rvx;\theta).
\vspace{-1ex}
\end{align*}

Combining NELBO with $\gL_{\mathrm{Anchor}}$, we get
$\gL_{\mathrm{T\text{-}ANELBO}}$.
\vspace{-0.5ex}
\begin{equation}
\begin{aligned}
&\gL_{\mathrm{T\text{-}ANELBO}}(\rvx;\theta)
=
\E_{q_{\sigma}(\rvz_{0:1}|\rvx)}
\left[
-\log p_\theta(\rvx|\rvz_0)
\right]
+
D_{\mathrm{KL}}
\left(
q_{\sigma}(\rvz_1|\rvx)
\,\Vert\,
p_\theta(\rvz_1)
\right)
\\
&\quad+
\sum_{i=1}^{T}
\E_{q_{\sigma}(\rvz_{0:1}|\rvx)}
\left[
D_{\mathrm{KL}}
\left(
q_{\sigma}(\rvz_{s(i)}|\rvz_{t(i)},\rvx)
\,\Vert\,
p_\theta(\rvz_{s(i)}|\rvz_{t(i)},\rvz_{t'(i)})
\right)
\right]
+
\gamma \gL_{\mathrm{Anchor}}(\rvx;\theta).
\end{aligned}
\label{eq:t-anelbo-kl}
\vspace{-1ex}
\end{equation}

where,
\begin{equation}
\gL_{\mathrm{Anchor}}(\rvx;\theta)
=
\sum_{i=1}^{T}
\E_{q_{\sigma}(\rvz_{0:1}|\rvx)}
\left[
D_{\mathrm{KL}}
\left(
r\!\left(
\rvy_{s'(i)}
\mid
\rvz_{t'(i)},\rvy
\right)
\,\Vert\,
r_{\theta_A}\!\left(
\rvy_{s'(i)}
\mid
\rvz_{t'(i)},
\rvy_{A_{\theta_A}}(\rvz_{t'(i)})
\right)
\right)
\right].
\label{eq:anchor-loss}
\end{equation}
Here $s'(i)=t'(i)-\frac{1}{T}$.

Thus, the T-ANELBO naturally decomposes into three terms:
\begin{itemize}[leftmargin=*,itemsep=0pt,topsep=1pt,parsep=0pt]
    \item \textbf{Reconstruction loss:}
    $\E_{q_{\sigma}(\rvz_{0:1}|\rvx)}
    \left[-\log p_\theta(\rvx|\rvz_0)\right]$,
    corresponding to reconstruction at the final reverse step.

    \item \textbf{Prior term:}
    $D_{\mathrm{KL}}
    \left(
    q_{\sigma}(\rvz_1|\rvx)
    \,\Vert\,
    p_\theta(\rvz_1)
    \right)$.
    This term vanishes when the terminal forward distribution matches the
    reverse-process prior, e.g., the all-mask state for absorbing diffusion.

    \item \textbf{Time-anchored diffusion and anchor loss:}
    \[
    \sum_{i=1}^{T}
    \E_{q_{\sigma}(\rvz_{0:1}|\rvx)}
    \left[
    D_{\mathrm{KL}}
    \left(
    q_{\sigma}(\rvz_{s(i)}|\rvz_{t(i)},\rvx)
    \,\Vert\,
    p_\theta(\rvz_{s(i)}|\rvz_{t(i)},\rvz_{t'(i)})
    \right)
    \right]
    +
    \gamma\gL_{\mathrm{Anchor}}(\rvx;\theta).
    \]
    The diffusion term trains the model to denoise the current state using
    an anchor obtained from the earlier, noisier state $\rvz_{t'(i)}$ helping
    in latent caching, while $\gL_{\mathrm{Anchor}}$ optionally supervises
    the anchor network's predicted anchor-token mixture.
\end{itemize}

Analyzing the diffusion and anchor losses together, we obtain:
\vspace{-1ex}
\begin{align*}
&\gL_{\mathrm{diffusion}}(\rvx;\theta)
+\gamma\gL_{\mathrm{Anchor}}(\rvx;\theta)
\\
&=
\sum_{i=1}^{T}
\E_{q_{\sigma}(\rvz_{0:1}|\rvx)}
\left[
D_{\mathrm{KL}}
\left(
q_{\sigma}(\rvz_{s(i)}|\rvz_{t(i)},\rvx)
\,\Vert\,
p_\theta(\rvz_{s(i)}|\rvz_{t(i)},\rvz_{t'(i)})
\right)
\right]
\\
&\quad+
\gamma\sum_{i=1}^{T}
\E_{q_{\sigma}(\rvz_{0:1}|\rvx)}
\left[
D_{\mathrm{KL}}
\left(
r\!\left(
\rvy_{s'(i)}|\rvz_{t'(i)},\rvy
\right)
\,\Vert\,
r_{\theta_A}\!\left(
\rvy_{s'(i)}|\rvz_{t'(i)},
\rvy_{A_{\theta_A}}(\rvz_{t'(i)})
\right)
\right)
\right]
\\
&=
\sum_{i=1}^{T}
\E_{q_{\sigma}(\rvz_{t(i)},\rvz_{t'(i)}|\rvx)}
\E_{q_{\sigma}(\rvz_{s(i)}|\rvz_{t(i)},\rvx)}
\left[
\log
\frac{
q_{\sigma}(\rvz_{s(i)}|\rvz_{t(i)},\rvx)
}{
p_\theta(\rvz_{s(i)}|\rvz_{t(i)},\rvz_{t'(i)})
}
\right]
\\
&\quad+
\gamma\sum_{i=1}^{T}
\E_{q_{\sigma}(\rvz_{t'(i)}|\rvx)}
\E_{r(\rvy_{s'(i)}|\rvz_{t'(i)},\rvy)}
\left[
\log
\frac{
r\!\left(
\rvy_{s'(i)}|\rvz_{t'(i)},\rvy
\right)
}{
r_{\theta_A}\!\left(
\rvy_{s'(i)}|\rvz_{t'(i)},
\rvy_{A_{\theta_A}}(\rvz_{t'(i)})
\right)
}
\right].
\end{align*}

We use the conditional independence step for the forward process that
\[
q_{\sigma}(\rvz_s|\rvz_{t(i)},\rvz_{t'(i)},\rvx)
=
q_{\sigma}(\rvz_{s(i)}|\rvz_{t(i)},\rvx)
\]
for $s<t\leq t'$ to arrive at the above equality.

We apply the two-case calculation in the proof of
Theorem~4.1 (Anchored Negative Evidence Lower Bound) of \citet{adlm},
given in Appendix~A.1 and restated there as Theorem~A.1.
For both the diffusion and anchor transitions, that calculation
shows that the KL contribution vanishes at unmasked positions
and reduces to a weighted cross-entropy at masked positions.

Applying this calculation to the diffusion transition at
$(s(i),t(i))$, with the denoiser conditioned additionally on
$\rvz_{t'(i)}$, gives
\begin{align*}
&\E_{q_{\sigma}(\rvz_{s(i)}|\rvz_{t(i)},\rvx)}
\left[
\log
\frac{
q_{\sigma}(\rvz_{s(i)}|\rvz_{t(i)},\rvx)
}{
p_\theta(\rvz_{s(i)}|\rvz_{t(i)},\rvz_{t'(i)})
}
\right]
\\
&\qquad =
\lambda_{t(i)}
\sum_{l=1}^{L}
\mathbf{1}_{\{\rvz_{t(i)}^l=\rvm\}}
\log\left\langle
\rvx_\theta^l(\rvz_{t(i)},\rvz_{t'(i)}),
\rvx^l
\right\rangle.
\end{align*}
The additional conditioning changes the predicted clean-token
distribution, while preserving the algebraic form of the
transition-level KL calculation.

Similarly, applying the anchor-transition calculation at
$(s'(i),t'(i))$ yields
\begin{align*}
&\E_{r(\rvy_{s'(i)}|\rvz_{t'(i)},\rvy)}
\left[
\log
\frac{
r\!\left(\rvy_{s'(i)}|\rvz_{t'(i)},\rvy\right)
}{
r_{\theta_A}\!\left(
\rvy_{s'(i)}|\rvz_{t'(i)},
\rvy_{A_{\theta_A}}(\rvz_{t'(i)})
\right)
}
\right]
\\
&\qquad =
\lambda_{t'(i)}
\sum_{l=1}^{L}
\mathbf{1}_{\{\rvz_{t'(i)}^l=\rvm\}}
\log\left\langle
\rvy^l_{A_{\theta_A}}(\rvz_{t'(i)}),
\rvy^l
\right\rangle,
\end{align*}
where
\begin{align*}
\lambda_{t(i)}
&=
\frac{
(1-\sigma_{t(i)})\alpha_{t(i)}-\alpha_{s(i)}
}{
1-\alpha_{t(i)}
},
\\
\lambda_{t'(i)}
&=
\frac{
(1-\sigma_{t'(i)})\alpha_{t'(i)}-\alpha_{s'(i)}
}{
1-\alpha_{t'(i)}
},
\qquad
s'(i)=t'(i)-\frac{1}{T}.
\end{align*}
Since these coefficients are nonpositive, each expression is a
nonnegative weight multiplying the corresponding cross-entropy.
Anchor supervision is restricted to the selected anchor positions.

Substituting both identities into the NELBO decomposition gives
\begin{align*}
&\gL_{\mathrm{T\text{-}ANELBO}}(\rvx;\theta)
\\
&=
\gL_{\mathrm{NELBO}}(\rvx;\theta)
+\gamma\gL_{\mathrm{Anchor}}(\rvx;\theta)
\\
&=
\E_{q_{\sigma}(\rvz_0|\rvx)}
\left[-\log p_\theta(\rvx|\rvz_0)\right]
+
D_{\mathrm{KL}}\!\left(
q_{\sigma}(\rvz_1|\rvx)
\,\Vert\,
p_\theta(\rvz_1)
\right)
\\
&\quad+
\sum_{i=1}^{T}
\E_{q_{\sigma}(\rvz_{t(i)},\rvz_{t'(i)}|\rvx)}
\left[
\lambda_{t(i)}
\sum_{l=1}^{L}
\mathbf{1}_{\{\rvz_{t(i)}^l=\rvm\}}
\log\left\langle
\rvx_\theta^l(\rvz_{t(i)},\rvz_{t'(i)}),
\rvx^l
\right\rangle
\right]
\\
&\quad+
\gamma\sum_{i=1}^{T}
\E_{q_{\sigma}(\rvz_{t'(i)}|\rvx)}
\left[
\lambda_{t'(i)}
\sum_{l=1}^{L}
\mathbf{1}_{\{\rvz_{t'(i)}^l=\rvm\}}
\log\left\langle
\rvy^l_{A_{\theta_A}}(\rvz_{t'(i)}),
\rvy^l
\right\rangle
\right].
\end{align*}

The prior term vanishes when
$p_\theta(\rvz_1)=q_{\sigma}(\rvz_1|\rvx)$.
The mask indicators are omitted when carry-over
parameterization is applied which makes the corresponding log-probability terms
zero at unmasked positions. We assume carry-over parameterization and hence these are zeroed-out.

Finally, the NELBO inequality gives
\begin{align*}
-\log p_\theta(\rvx)
+\gamma\gL_{\mathrm{Anchor}}(\rvx;\theta)
\leq
\gL_{\mathrm{T\text{-}ANELBO}}(\rvx;\theta).
\end{align*}

\end{proof}

\textbf{Implications.}
The T-ANELBO objective highlights two aspects of time-based anchoring:
\vspace{-1ex}
\begin{itemize}
    \item \textbf{Temporal latent reuse.}
    The term
    \[
    \sum_{i=1}^{T}
    \E_{q(\rvz_{t(i)},\rvz_{t'(i)}|\rvx)}
    \left[
    \sum_{l=1}^{L}
    \lambda_{t(i)}
    \log
    \left\langle
    \rvx^l_{\theta}
    (\rvz_{t(i)},\rvz_{t'(i)}),
    \rvx^l
    \right\rangle
    \right]
    \]
    trains the denoising prediction
    $\rvx_{\theta}(\rvz_{t(i)},\rvz_{t'(i)})$ using the latent representation
    computed from the earlier, noisier anchor state $\rvz_{t'(i)}$, where
    $t'(i)\geq t(i)$. By sampling stale anchor times during training, the
    objective explicitly exposes the model to temporally stale latent
    representations and encourages them to remain useful across multiple
    diffusion steps. This enables latent-space reuse and, consequently,
    anchor caching at inference time.

    \item \textbf{Optional anchor supervision.}
    The term
    \[
    \sum_{i=1}^{T}
    \E_{q(\rvz_{t(i)},\rvz_{t'(i)}|\rvx)}
    \left[
    \sum_{l=1}^{L}
    \gamma\lambda_{t'(i)}
    \log
    \left\langle
    \rvy^l_{A_{\theta_A}}(\rvz_{t'(i)}),
    \rvy^l
    \right\rangle
    \right]
    \]
    directly supervises the anchor-network prediction
    $\rvy_{A_{\theta_A}}(\rvz_{t'(i)})$ when explicit anchor-token targets
    $\rvy$ are available. Its contribution is controlled by $\gamma$.
    Setting $\gamma=0$ removes the need for predefined anchor targets, in
    which case the latent anchor is learned entirely through the temporal
    denoising objective; for $\gamma>0$, known anchor tokens can
    additionally guide the learned anchor representation.
\end{itemize}
    
\section{Related Work}
\subsection{Discrete Diffusion Models}
Discrete Diffusion Language Models~\citep{sohl2015deep,d3pm} follow a noising process gradually corrupting the clean input sequence $\rvx = (\rvx^1, \rvx^2, \cdots, \rvx^L)$ to a partially noised
sequence $\rvz_t = (\rvz_t^1, \rvz_t^2, \cdots, \rvz_t^L)$ and finally to the noisy prior $\boldsymbol{\pi}$. Let $T$ be the number of denoising steps, we denote $t(i)=\frac{i}{T}$ and $s(i)=\frac{i-1}{T}$.
For the rest of the paper, we denote $t(i)$ as $t$ and $s(i)$ as $s$. The forward noising process, as described in D3PM~\citep{d3pm}: 
\vspace{-1ex}
\begin{align}
\label{eq:fwd-old}
    q(\rvz_t|\rvx) = \prod_{l=1}^{L} q(\rvz_t^l|\rvx), \quad   q(\rvz_t^l|\rvx) = \cat\left(\rvz_t^l; \alpha_t \rvx^l + (1-\alpha_t) \boldsymbol{\pi} \right),\quad l \in \{1,2,\cdots, L\},
\vspace{-1ex}
\end{align}
where $\alpha_t \in [0,1]$ is a monotonically decreasing noise schedule, with $\alpha_0=1$ and $\alpha_1=0$. We define $\alpha_{t|s}=\frac{\alpha_t}{\alpha_s}$. The one step transition 
probability is $q(\rvz_t^l|\rvz_s^l) = \cat(\rvz_t^l; \alpha_{t|s}\rvz_s^l + (1-\alpha_{t|s})\boldsymbol{\pi})$. The reverse posterior is given as:
\vspace{-1ex}
\begin{equation}
\label{eq:rev-posterior-old}
q(\rvz_s^l\mid\rvz_t^l, \rvx^l)
=
\operatorname{Cat}\!\left(
\rvz_s^l;
\frac{
\left[
    \alpha_{t \mid s}\rvz_t^l
    +
    \left(1-\alpha_{t \mid s}\right)
    \mathbf{1}\boldsymbol{\pi}^{\top}\rvz_t^l
\right]
\odot
\left[
    \alpha_s\rvx^l
    +
    \left(1-\alpha_s\right)\boldsymbol{\pi}
\right]
}{
    \alpha_t\rvz_t^{l^{\top}}\rvx^l
    +
    \left(1-\alpha_t\right)
    \rvz_t^{l^{\top}}\boldsymbol{\pi}
}
\right).
\vspace{-1ex}
\end{equation}

During inference, we do not have access to the clean samples $\rvx$, therefore we approximate it with a neural network $\theta$ which learns to predict the distribution of the clean sequence 
from the noisy input. The distribution predicted by $\theta$ is $\rvx_{\theta}(\rvz_t) = (\rvx_{\theta}^1(\rvz_t), \rvx_{\theta}^2(\rvz_t) \cdots, \rvx_{\theta}^L(\rvz_t))$ where 
$\rvx_{\theta}^l(\rvz_t)$ is the probability distribution over $\gV$. We define $p_{\theta}(\rvx, \rvz_{0:1})$ as the learnt joint distribution over the sequences which follows a markovian structure, 
that is, $p_\theta(\rvx,\rvz_{0:1})=
p_\theta(\rvz_1) p_\theta(\rvx|\rvz_0)\prod_{i=1}^{T} p_\theta(\rvz_{s(i)} | \rvz_{t(i)})$. Also, each token distribution in the sequence is conditionally independent given 
$\rvz_t$, i.e, $p_\theta(\rvz_s |\rvz_t) = \prod_{l=1}^{L} p_\theta(\rvz_s^l | \rvz_t)$. We can represent the probability assigned by the model corresponding to the true token at position $l$ as 
$p_{\theta}(\rvx^l|\rvz_t) = \langle\rvx_{\theta}^l(\rvz_t), \rvx^l \rangle$. In the discrete-time setting, the denoising model is trained by minimizing
the standard negative evidence lower bound (NELBO),
\vspace{-1ex}
\begin{align}
\gL_{\mathrm{NELBO}}(\rvx,\theta)
=&\;
\E_q\left[-\log p_\theta(\rvx\mid \rvz_{t(0)})\right]
\nonumber +
\sum_{i=1}^{T}
\E_q\left[
D_{\mathrm{KL}}\left(
q(\rvz_{s(i)}\mid \rvz_{t(i)},\rvx)
\,\Vert\,
p_\theta(\rvz_{s(i)}\mid \rvz_{t(i)})
\right)
\right]
\nonumber\\
&+
D_{\mathrm{KL}}\left(
q(\rvz_{t(T)}\mid\rvx)
\,\Vert\,
p_\theta(\rvz_{t(T)})
\right).
\label{eq:discrete-nelbo-old}
\vspace{-1ex}
\end{align}
We focus on two corruption processes relevant to our methods: absorbing-state masked diffusion~\citep{mdlm,remdm, md4}, which underlies our TADM pretraining experiments, and uniform-state diffusion~\citep{schiff2025simple,Duo}, which underlies the pretrained model used for TADM:Post-train.

\subsubsection{Masked Diffusion Models}
MDMs use masked token prior i.e. $\boldsymbol{\pi}=\rvm$, a one-hot vector at the special MASK token index, which simplifies Eq. \ref{eq:rev-posterior} to 
\vspace{-1ex}
\begin{align}
    \label{eq:masked-posterior-known}         
     q(\rvz_s^l | \rvz_t^l, \rvx^l) 
     =
    \begin{cases}
    \mathrm{Cat}(\rvz_s^l; \rvz_t^l), &  \rvz_t^l \neq \rvm \\
    \mathrm{Cat}\left(\rvz_s^l; \frac{\alpha_s - \alpha_t}{1 - \alpha_t} \rvx^l + \frac{1 - \alpha_s}{1 - \alpha_t} \rvm \right),&\rvz_t^l = \rvm.
    \end{cases}
\vspace{-1ex}
\end{align}
\textbf{MDLM}~\citep{mdlm} specialize this construction by imposing two constraints on the output probability distribution.
They are: \textit{zero-masking}, which forces the model to assign zero probability to the mask token, i.e. $\langle \rvx_{\theta}^l(\rvz_t), \rvm \rangle = 0$ and \textit{carry-over unmasking}, which states 
once a token is decoded, it will remain the same throughout the generation process, meaning if $\rvz_t^l \neq \rvm$ then $\langle \rvx_{\theta}^l(\rvz_t), \rvz_t^l \rangle = 1$. 
Equipped with this, we can formulate the reverse generation step as:
\vspace{-1ex}
\begin{align}
    \label{eq:mdlm-posterior-unk}    
     p_\theta(\rvz_s^l|\rvz_t)
     =
     q(\rvz_s^l\mid\rvz_t^l, \rvx_{\theta}^l(\rvz_t))
     =
    \begin{cases}
    \mathrm{Cat}(\rvz_s^l; \rvz_t^l), &  \rvz_t^l \neq \rvm \\
    \mathrm{Cat}\left(\rvz_s^l; \frac{\alpha_s - \alpha_t}{1 - \alpha_t} \rvx^l_\theta(\rvz_t) + \frac{1 - \alpha_s}{1 - \alpha_t} \rvm \right),&\rvz_t^l = \rvm.
    \end{cases}
    \vspace{-1ex}
\end{align}
The $\theta-$parametrized model is trained on the NELBO objective. The loss for this process is formulated below:
\vspace{-1ex}
\begin{align}
\label{eq:masked-nelbo} 
&\gL_{NELBO}(\rvx, \theta)
=
\E_{Z_0 \sim q(\cdot | \rvx)}\Big[-\log p_\theta(\rvx |Z_0) \Big] \\
&\hspace{4em}+
\sum_{i=1}^{T}\E_{Z_{t(i)} \sim q(\cdot|\rvx)}\Bigg[\frac{\alpha_{t(i)} - \alpha_{s(i)}}{1-\alpha_{t(i)}} \sum_{l=1}^{L} \log\langle\rvx^l_\theta(Z_{t(i)}), \rvx^l \rangle \Bigg].
\vspace{-1ex}
\end{align}

\textbf{ReMDM}~\citep{remdm} identifies a limitation of standard
absorbing-state diffusion: once a token is unmasked during inference, it
cannot be changed again. Consequently, early decoding errors cannot be
corrected. ReMDM introduces a remasking reverse process that assigns
nonzero probability to returning an already decoded token to the mask
state, enabling iterative error correction during generation.
The ReMDM posterior is defined as
\begin{equation}
q_{\sigma}(\rvz_s^l \mid \rvz_t^l,\rvx^l)
=
\begin{cases}
\operatorname{Cat}\!\left(
\rvz_s^l;
(1-\sigma_t)\rvx^l+\sigma_t\rvm
\right),
& \rvz_t^l \neq \rvm, \\[1ex]
\operatorname{Cat}\!\left(
\rvz_s^l;
\frac{\alpha_s-(1-\sigma_t)\alpha_t}{1-\alpha_t}\rvx^l
+
\frac{1-\alpha_s-\sigma_t\alpha_t}{1-\alpha_t}\rvm
\right),
& \rvz_t^l = \rvm ,
\end{cases}
\label{eq:remdm-posterior}
\end{equation}
where $\sigma_t$ controls the probability of remasking an already
decoded token. Setting $\sigma_t=0$ recovers the standard masked
diffusion posterior.

At inference time, the clean token $\rvx^l$ is unknown and is replaced
with the denoiser prediction $\rvx_\theta^l(\rvz_t)$. The learned reverse
transition is therefore parameterized as
\begin{equation}
p_\theta(\rvz_s^l \mid \rvz_t)
=
q_{\sigma}\!\left(
\rvz_s^l
\mid
\rvz_t^l,
\rvx_\theta^l(\rvz_t)
\right),
\end{equation}
or explicitly,
\begin{equation}
p_\theta(\rvz_s^l \mid \rvz_t)
=
\begin{cases}
\operatorname{Cat}\!\left(
\rvz_s^l;
(1-\sigma_t)\rvx_\theta^l(\rvz_t)+\sigma_t\rvm
\right),
& \rvz_t^l \neq \rvm, \\[1ex]
\operatorname{Cat}\!\left(
\rvz_s^l;
\frac{\alpha_s-(1-\sigma_t)\alpha_t}{1-\alpha_t}
\rvx_\theta^l(\rvz_t)
+
\frac{1-\alpha_s-\sigma_t\alpha_t}{1-\alpha_t}\rvm
\right),
& \rvz_t^l = \rvm .
\end{cases}
\label{eq:remdm-param}
\end{equation}

With this parameterization, ReMDM minimizes the negative evidence lower
bound
\begin{align}
\gL_{\mathrm{ReMDM}}(\rvx;\theta)
=&\;
\E_{\rvz_0\sim q(\cdot\mid\rvx)}
\left[
-\log p_\theta(\rvx\mid\rvz_0)
\right]
\nonumber\\
&+
\E_{t\sim\{1/T,\ldots,1\}}
\E_{\rvz_t\sim q(\rvz_t\mid\rvx)}
\left[
T\,
\frac{(1-\sigma_t)\alpha_t-\alpha_s}
{1-\alpha_t}
\log
\left\langle
\rvx_\theta(\rvz_t),\rvx
\right\rangle
\right].
\label{eq:remdm-nelbo}
\end{align}

\subsubsection{Uniform-State Diffusion Models}
Alternative to MDM, USDMs set the prior to uniform noise, i.e. $\boldsymbol{\pi}=\mathbf{u}:=\frac{1}{K}\mathbf{1}$. Here $K$ is the vocabulary size. Its reverse posterior therefore becomes:
\vspace{-1ex} 
\begin{equation}
\label{eq:usdm-rev-posterior-known}
q(\rvz_s^\ell\mid\rvz_t^\ell,\rvx^\ell)
=
\operatorname{Cat}\!\biggl(
\rvz_s^\ell;\,
\frac{
K\alpha_t\rvz_t^\ell\odot\rvx^\ell
+(\alpha_{t\mid s}-\alpha_t)\rvz_t^\ell
+(\alpha_s-\alpha_t)\rvx^\ell
+\tfrac{(\alpha_s-\alpha_t)(1-\alpha_s)}{K\alpha_s}\mathbf{1}
}{
K\alpha_t\langle\rvz_t^\ell,\rvx^\ell\rangle+1-\alpha_t
}
\biggr)
\vspace{-1ex}
\end{equation}

The denoising processes uses $p_\theta(\rvz_s^l|\rvz_t)=q(\rvz_s^l|\rvz_t^l, \rvx_{\theta}^l(\rvz_t))$, with which we get the following parametrization:
\vspace{-1ex} 
\begin{equation}
    \label{eq:usdm-rev-posterior-unk}
p_\theta(\rvz_s^\ell\mid\rvz_t)
=
\operatorname{Cat}\!\biggl(
\rvz_s^\ell;\,
\frac{
K\alpha_t\rvz_t^\ell\odot\rvx_\theta^\ell
+(\alpha_{t\mid s}-\alpha_t)\rvz_t^\ell
+(\alpha_s-\alpha_t)\rvx_\theta^\ell
+\tfrac{(\alpha_s-\alpha_t)(1-\alpha_s)}{K\alpha_s}\mathbf{1}
}{
K\alpha_t\langle\rvz_t^\ell,\rvx_\theta^\ell\rangle+1-\alpha_t
}
\biggr)
\vspace{-1ex} 
\end{equation}
where $\rvx_{\theta}^l = \rvx_\theta^\ell(\rvz_t,t)$.
Following \citet{schiff2025simple}, this objective admits a tighter
continuous-time formulation as $T\rightarrow\infty$. For a uniform
limiting distribution and a noise schedule satisfying
$\alpha_0=1$ and $\alpha_1=0$, the reconstruction and prior terms vanish,
leaving only the diffusion loss.

For each token position $l$, define
    \vspace{-1ex} 
\begin{align}
\bar{\rvx}_t^l
&=
K\alpha_t\rvx^l
+
(1-\alpha_t)\mathbf{1},
\\
\bar{\rvx}_{\theta,t}^l
&=
K\alpha_t\rvx_\theta^l(\rvz_t,t)
+
(1-\alpha_t)\mathbf{1},
\vspace{-1ex} 
\end{align}
and let
\vspace{-1ex} 
\begin{equation}
i_l
=
\arg\max_{j\in[K]}
\rvz_t^l[j]
\vspace{-1ex} 
\end{equation}
denote the index of the observed token at position $l$ in $\rvz_t$.
The continuous-time USDM NELBO is then
\vspace{-1ex} 
\begin{align}
\gL_{\mathrm{USDM}}^{\infty}(\rvx,\theta)
=
\int_0^1
\E_{\rvz_t\sim q(\cdot\mid\rvx)}
\sum_{l=1}^{L}
\frac{\alpha_t'}{K\alpha_t}
\Bigg[
&
\frac{K}{\bar{\rvx}_t^l[i_l]}
-
\frac{K}{\bar{\rvx}_{\theta,t}^l[i_l]}
\nonumber\\
&
-
\sum_{\substack{j\in[K]\\ \rvz_t^l[j]=0}}
\frac{\bar{\rvx}_t^l[j]}
     {\bar{\rvx}_t^l[i_l]}
\log
\left(
\frac{
\bar{\rvx}_{\theta,t}^l[i_l]\,
\bar{\rvx}_t^l[j]
}{
\bar{\rvx}_{\theta,t}^l[j]\,
\bar{\rvx}_t^l[i_l]
}
\right)
\Bigg]
\,dt .
\label{eq:usdm-continuous-nelbo}
\vspace{-1ex} 
\end{align}

\subsection{Anchored Diffusion Language Models}
\label{sec:ADLM-app}
MDMs have their own limitations, we do not have control over which tokens will be unmasked during the generation step and hence, its possible that the important tokens may be unmasked very late in the generation step, thereby hurting the sample quality. Here, the important tokens are those, with the help of which the model can learn to produce higher quality text. Hence, in order to deal with this, \citet{adlm} introduced a two stage method, known as the \textit{Anchored Diffusion Language Models} (ADLM).
The core intuition of ADLM is that in the presence of the important tokens which we will call the \textit{anchor tokens}, the denoiser can generate higher quality text. ADLM is decomposed into two stage network, the first stage is the \textit{anchor network}, which produces soft context logits. Conditioned on these, the \textit{denoiser network} outputs the final clean token sequence. Therefore the output prediction is now a composition of two network: 
\vspace{-1ex} 
\begin{equation}
    \rvx_\theta(\rvz_t)
    =
    \rvx_{\theta_D}\bigl(\rvy_{\theta_A}(\rvz_t)\bigr),
    \vspace{-1ex} 
\end{equation}
Here the $\theta$ parametrized model is split into two models $\theta = [\theta_{A}, \theta_D]$, corresponding to the anchor and the denoiser networks respectively. Both of them output probability distribution on the vocabulary space $\gV$. \\
\textbf{Reverse inference posterior for the main process.}
Following \citet{adlm}, the reverse transition uses the clean-token
prediction $\rvx^l_{\theta}(\rvz_t)$:
\begin{equation}
\begin{aligned}
p_\theta(\rvz_s^l\mid\rvz_t)
 =
\begin{cases}
\mathrm{Cat}\!\left(
\rvz_s^l;
(1-\sigma_t)\rvz_t^l+\sigma_t\rvm
\right), \rvz_t^l\neq\rvm,
\\[1ex]
\mathrm{Cat}\!\left(
\rvz_s^l;
\dfrac{\alpha_s-(1-\sigma_t)\alpha_t}{1-\alpha_t}
\rvx^l_{\theta}(\rvz_t)
+\dfrac{1-\alpha_s-\alpha_t\sigma_t}{1-\alpha_t}\rvm
\right),
\rvz_t^l=\rvm.
\end{cases}
\end{aligned}
\label{eq:app-adlm-inf-post}
\end{equation}
Here $t=t(i)$, $s=s(i)$, and $\sigma_t$ denotes the remasking
probability. An unmasked token is carried over with probability
$1-\sigma_t$ and remasked with probability $\sigma_t$. At a masked
position, the model uses the anchor-conditioned denoiser prediction.
Setting $\sigma_t=0$ recovers the absorbing-state reverse transition.

\textbf{Reverse inference posterior for the anchor transition.}
Let $\rvy=\gA(\rvx)$ denote the clean anchor-token sequence.
The learned anchor transition is
\begin{equation}
\begin{aligned}
&r_{\theta_A}\!\left(
\rvy_s^l\mid\rvz_t,\rvy_{A_{\theta_A}}(\rvz_t)
\right) =
\begin{cases}
\mathrm{Cat}\!\left(
\rvy_s^l;
(1-\sigma_t)\rvy^l+\sigma_t\rvm
\right),
 \rvz_t^l\neq\rvm,
\\[1ex]
\mathrm{Cat}\!\left(
\rvy_s^l;
\dfrac{\alpha_s-(1-\sigma_t)\alpha_t}{1-\alpha_t}
\rvy^l_{A_{\theta_A}}(\rvz_t)
+\dfrac{1-\alpha_s-\alpha_t\sigma_t}{1-\alpha_t}\rvm
\right),
 \rvz_t^l=\rvm.
\end{cases}
\end{aligned}
\label{eq:app-adlm-anchor-transition}
\end{equation}
For an unmasked supervised anchor position, the anchor token is known
from the observed token and is carried over. At a masked position,
the anchor network predicts its distribution. The corresponding
target transition uses the clean anchor token:
\begin{equation}
\begin{aligned}
&r(\rvy_s^l\mid\rvz_t^l,\rvy^l) =
\begin{cases}
\mathrm{Cat}\!\left(
\rvy_s^l;
(1-\sigma_t)\rvy^l+\sigma_t\rvm
\right),
 \rvz_t^l\neq\rvm,
\\[1ex]
\mathrm{Cat}\!\left(
\rvy_s^l;
\dfrac{\alpha_s-(1-\sigma_t)\alpha_t}{1-\alpha_t}\rvy^l
+\dfrac{1-\alpha_s-\alpha_t\sigma_t}{1-\alpha_t}\rvm
\right),
 \rvz_t^l=\rvm.
\end{cases}
\end{aligned}
\label{eq:app-adlm-anchor-target}
\end{equation}

The sequence-level transitions factorize over token positions:
\begin{equation}
\begin{aligned}
r(\rvy_s\mid\rvz_t,\rvy)
&=
\prod_{l=1}^{L}r(\rvy_s^l\mid\rvz_t^l,\rvy^l),
\\
r_{\theta_A}\!\left(
\rvy_s\mid\rvz_t,\rvy_{A_{\theta_A}}(\rvz_t)
\right)
&=
\prod_{l=1}^{L}r_{\theta_A}\!\left(
\rvy_s^l\mid\rvz_t,\rvy_{A_{\theta_A}}(\rvz_t)
\right).
\end{aligned}
\label{eq:app-adlm-anchor-factorization}
\end{equation}

\textbf{Anchor loss.}
ADLM aligns the learned anchor transition with the target transition
using the expected sequence-level KL divergence:
\begin{equation}
\begin{aligned}
&\gL_{\mathrm{Anchor}}(\rvx;\theta)
=
\sum_{i=1}^{T}
\E_{q(\rvz_{0:1}|\rvx)}
\Biggl[
D_{\mathrm{KL}}\Bigl(
r\!\left(\rvy_{s(i)}\mid\rvz_{t(i)},\rvy\right)
\\
&\hspace{8em}\Vert\,
r_{\theta_A}\!\left(
\rvy_{s(i)}\mid\rvz_{t(i)},
\rvy_{A_{\theta_A}}(\rvz_{t(i)})
\right)
\Bigr)
\Biggr].
\end{aligned}
\label{eq:app-adlm-anchor-loss}
\end{equation}
The KL divergence is over the full anchor sequence. Under
Eq.~\ref{eq:app-adlm-anchor-factorization}, it equals the sum of
the tokenwise KL divergences. The target and learned transitions
coincide at unmasked positions, yielding zero KL contribution.
Anchor supervision is restricted to the selected important-token
positions in practice.

Our cached anchor loss in Eq.~\ref{eq:anchor-loss} evaluates these
transitions at $(s'(i),t'(i))$, where $s'(i)=t'(i)-1/T$.
ADLM corresponds to the fresh-anchor case $t'(i)=t(i)$.
The training objective for this model is:

\begin{align}
\label{eq:anelbo-old}
&\gL_{\mathrm{ANELBO}}(\rvx; \theta_A, \theta_D) 
\coloneq 
\E_{Z_0 \sim q(\cdot|\rvx)}\left[-\log p_{\theta_D}(\rvx | \rvy_{\theta_A}(Z_0))\right] \\
&\hspace{2em}+ 
\sum_{i=1}^{T}\E_{Z_{t(i)} \sim q(\cdot|\rvx)} 
\left[
\frac{\alpha_{t(i)} - \alpha_{s(i)}}{1-\alpha_{t(i)}} 
\sum_{l=1}^L \left(
\log\langle \rvx^l_{\theta_D}(\rvy_{\theta_{A}}(Z_{t(i)})), \rvx^l\rangle + \gamma \log\langle \rvy^l_{\theta_A}(Z_{t(i)}), \rvy^l\rangle
\right)
\right],
\vspace{-1ex} 
\end{align}
Here $\gamma$ is an hyperparameter which controls the anchor supervision factor. The network is trained on the standard NELBO according to \ref{eq:discrete-nelbo} in addition to an extra loss which is coined as the \textit{anchor loss} to train the anchor network produce better anchor tokens.

\textbf{Anchoring and Conditional Uncertainty.}
The motivation behind anchoring is that not all tokens provide equal information for reconstructing a sequence. Let $X_A$ denote a subset of anchor tokens and $X_{\bar A}$ denote the remaining tokens. Given an anchor budget $d$, an ideal set of anchors can be viewed as one that minimizes the uncertainty of the remaining sequence,
\begin{equation}
    A^\star \in \operatorname*{arg\,min}_{A\subseteq[L],\,|A|\leq d}
    H\!\left(X_{\bar A}\mid X_A\right).
    \label{eq:adlm-conditional-entropy}
\end{equation}
Hence, revealing or accurately predicting $X_A$ makes the remaining variables easier to estimate. ADLM realizes this principle by training the anchor network to predict informative tokens early and allowing the denoiser to condition on these predictions.

This reduction in conditional uncertainty also motivates the statistical benefit of anchoring. Under the graphical-model analysis of \citet{adlm}, suppose each token is categorical and each masked token depends on an anchor set of bounded size $|\pi_l|\leq d$, with $d\ll L$. A standard diffusion model may condition a token on the remaining sequence, requiring $O(V^L)$ parameters per conditional in the corresponding tabular model and a total sample complexity of $O(LV^L)$. Conditioning instead on a bounded set of anchors reduces the per-token complexity to $O(V^{d+1})$, yielding
\begin{equation}
    O(LV^L)
    \quad\longrightarrow\quad
    O(LV^{d+1}).
    \label{eq:adlm-sample-complexity}
\end{equation}
Thus, when a small set of informative anchors captures the relevant dependencies of the sequence, anchoring reduces the effective conditioning dimension and can substantially reduce the samples required to learn the corresponding conditional distributions.

\section{Experimental Setup}
In this section, we report the experimental setup used for TADM:Post-train and TADM:Pretraining along with their hyperparameters.
\subsection{TADM:Post-train}
\label{sec:app-tadm-post-train}
We used the DiffusionGemma~\citep{diffusiongemma}, a 26B MoE Uniform-State Block Diffusion Language Model as the base model to demonstrate the capability of our framework. 

\textbf{Architecture Details.} Diffusion Gemma consists of 30 text layers, along with an embedding and output head. It also has a vision backbone which is frozen throughout the experiment. The model acts as an encoder and decoder, with most of the weights shared between them. The encoder stage prefills the context using causal attention. It first runs on the prompt and builds the KV cache, which acts as a context during the denoising process. The decoder iteratively denoises the canvas and uses bidirectional attention and once finalized is appended into the context by the encoder. 
For our experiment, we split the base model into three components and introduce a fusion module.
\begin{itemize}[
    leftmargin=*,
    labelsep=0.4em,
    itemsep=0pt,
    parsep=0pt,
    topsep=0pt,
    partopsep=0pt
]
    \item \textbf{Shared Network.} It consists of the embedding and the first two text layer. This network is computed every step producing the current state $\rvc_t$. Its purpose is to introduce an alignment of the latent space of the current state features with the anchor cache.
    \item \textbf{Anchor Network.} This is the next 20 text layers. It is the expensive network which contains the deep latent representations which will be stored as the anchor cache $\rvh_t'$.
\item \textbf{Fusion Module.}
The post-training fusion module uses \emph{paired attention}. In addition
to the cached deep anchor $\mathbf{H}_0$, we retain the shallow
representation $\mathbf{C}_0$ computed at the same anchor-refresh step.
This cached shallow state is an implementation detail omitted from the
general formulation in Sec.~\ref{sec:stale-anchor}; it provides a reference
for the change in the current canvas and ensures that the fusion correction
vanishes exactly when the current and cached shallow states coincide.

For a sequence of length $L$, let
$\mathbf{C},\mathbf{C}_0,\mathbf{H}_0\in\mathbb{R}^{L\times d}$ denote
the current shallow representation, cached shallow representation, and
cached deep anchor, respectively. We first compute tokenwise RMS scales
\begin{equation}
s_{C,i}
=
\sqrt{\frac{1}{d}\sum_{j=1}^{d}C_{0,ij}^{2}+\epsilon},
\qquad
s_{H,i}
=
\sqrt{\frac{1}{d}\sum_{j=1}^{d}H_{0,ij}^{2}+\epsilon},
\end{equation}
with $\epsilon=10^{-6}$, and project to rank $r$:
\begin{equation}
\mathbf{Z}
=
(\mathbf{C}/\mathbf{s}_C)W_C,
\qquad
\mathbf{Z}_0
=
(\mathbf{C}_0/\mathbf{s}_C)W_C,
\qquad
\mathbf{U}
=
(\mathbf{H}_0/\mathbf{s}_H)W_H .
\end{equation}
Both shallow branches use the cached-state scale $\mathbf{s}_C$ and the
same projection $W_C$.

We define a shared mixer $\phi$ by first concatenating the shallow and
anchor features,
\begin{equation}
\mathbf{J}=[\mathbf{Z},\mathbf{U}],
\qquad
\mathbf{Q}=\mathbf{J}W_Q,\quad
\mathbf{K}=\mathbf{J}W_K,\quad
\mathbf{V}=\mathbf{J}W_V,
\end{equation}
and applying multi-head self-attention followed by a residual FFN:
\begin{equation}
\mathbf{R}
=
\mathbf{Z}
+
\operatorname{MHA}(\mathbf{Q},\mathbf{K},\mathbf{V})W_O,
\qquad
\phi(\mathbf{Z},\mathbf{U})
=
\mathbf{R}
+
\operatorname{SiLU}(\mathbf{R}W_1)W_2 .
\end{equation}
The attention is bidirectional over the current canvas; in our experiments
$r=256$ with four heads, and the FFN expands $r\rightarrow2r\rightarrow r$.

The paired-attention correction is the difference between two evaluations
of the same mixer under identical anchor context:
\begin{equation}
\boxed{
\mathbf{M}_{t,t'}
=
\phi(\mathbf{Z},\mathbf{U})
-
\phi(\mathbf{Z}_0,\mathbf{U})
}.
\label{eq:paired-attention-difference}
\end{equation}
We then form
\begin{equation}
\mathbf{F}
=
[\mathbf{U};\mathbf{Z}_0;\mathbf{Z}-\mathbf{Z}_0;\mathbf{M}_{t,t'}],
\qquad
\mathbf{g}
=
\sigma\!\left(
\operatorname{SiLU}(\mathbf{F}W_{g1})W_{g2}+b_g
\right),
\end{equation}
where $\mathbf{g}\in\mathbb{R}^{L\times1}$ is a scalar gate for each
token. The fused deep representation is
\begin{equation}
\boxed{
\widetilde{\mathbf{H}}_{t\mid t'}
=
\mathbf{H}_0
+
\mathbf{g}\odot\mathbf{s}_H\odot
\left(\mathbf{M}_{t,t'}W_{\mathrm{up}}\right)
}.
\label{eq:stale-anchor-fusion-app}
\end{equation}

Importantly, when $\mathbf{C}=\mathbf{C}_0$, the shared normalization and
projection imply $\mathbf{Z}=\mathbf{Z}_0$. Because both branches use the
same mixer parameters and the same anchor context $\mathbf{U}$,
\begin{equation}
\mathbf{C}=\mathbf{C}_0
\;\Longrightarrow\;
\phi(\mathbf{Z},\mathbf{U})
=
\phi(\mathbf{Z}_0,\mathbf{U})
\;\Longrightarrow\;
\mathbf{M}_{t,t'}=\mathbf{0},
\end{equation}
and therefore
\begin{equation}
\widetilde{\mathbf{H}}_{t\mid t'}=\mathbf{H}_0 .
\end{equation}
Thus, at a fresh anchor the fusion module exactly preserves the cached
deep representation, while for stale anchors it learns the change induced
by the evolving current canvas.
    \item \textbf{Denoiser Network.} This is the final 8 text layers and the output head of Diffusion Gemma. It takes in the fused output $\widetilde{\mathbf{h}}_{t\mid t'}$ and then predicts the output distribution of the tokens.
\end{itemize}
The post training method does not have $\gamma$ supervision as attaching a LM-Head to the output of anchor network requires expensive training.

\begin{algorithm}[t]
\caption{Anchor-Cache Inference}
\label{alg:anchor-cache-inference}
\KwIn{
Initial noisy state $\rvz_T$;
number of reverse steps $T$;
anchor refresh interval $K$;
shared network $S_{\theta_S}$;
anchor network $A_{\theta_A}$;
denoiser $D_{\theta_D}$;
fusion module $\Phi_{\phi}$;
diffusion sampler $\mathcal{P}$
}
\KwOut{Generated sequence $\rvz_0$}
$\mathbf{h}^{\mathrm{cache}} \leftarrow \varnothing$\;
$j \leftarrow \varnothing$\;

\For{$i \leftarrow T$ \KwTo $1$}{
    $\mathbf{c}_i \leftarrow S_{\theta_S}(\rvz_i)$\;

    \If{$(T-i) \bmod K = 0$}{
        $\mathbf{h}^{\mathrm{cache}}
        \leftarrow A_{\theta_A}(\mathbf{c}_i)$\;
        $j \leftarrow i$\;
    }

    $\widetilde{\mathbf{h}}_i
    \leftarrow
    \Phi_{\phi}
    \left(
        \mathbf{c}_i,
        \mathbf{h}^{\mathrm{cache}}
    \right)$\;

    $\widehat{\rvx}_i
    \leftarrow
    D_{\theta_D}
    \left(
        \widetilde{\mathbf{h}}_i
    \right)$\;

    $\rvz_{i-1}
    \sim
    \mathcal{P}
    \left(
        \rvz_i,
        \widehat{\rvx}_i,
        i
    \right)$\;
}
\Return{$\rvz_0$}\;
\end{algorithm}

\textbf{Sampler Details.}
We use Algorithm~\ref{alg:anchor-cache-inference} together with the
official DiffusionGemma entropy-bound sampler, summarized in
Algorithm~\ref{alg:entropy-bound-sampler}. At each denoising step, the
sampler ranks token positions by predictive entropy, accepts the most
confident positions subject to the entropy bound, and re-noises the
remaining positions for further refinement.

The sampler configuration is:
\begin{itemize}[leftmargin=*,itemsep=0pt,topsep=2pt]
    \item \textbf{Entropy bound:} $\eta=0.1$.
    \item \textbf{Temperature schedule:} annealed from $0.8$ to $0.4$.
    \item \textbf{Maximum denoising steps:} 48.
    \item \textbf{Canvas block size:} 256 tokens.
    \item \textbf{Anchor refresh interval:} $K\in\{1,2,3\}$, where
    $K=1$ refreshes the anchor at every reverse step.
    \item \textbf{Self-conditioning:} enabled between reverse steps.
    \item \textbf{Adaptive early stopping:} disabled in our experiments.
    \item \textbf{Evaluation mode:} all benchmark evaluations use
    DiffusionGemma's \emph{no-think} mode.
\end{itemize}

Although the released DiffusionGemma sampler supports adaptive early stopping,
we disable it in all evaluations. DiffusionGemma reports using sampler
distillation with reinforcement learning (SD-RL) to compress the denoising
trajectory, this greatly helps in the adaptive stopping aspect of the sampler, whereas our objective is complementary: we study whether the cost
of a fixed-step diffusion sampler can be reduced by reusing stale latent
representations. We therefore evaluate both DiffusionGemma and
TADM:Post-train with adaptive stopping disabled and a fixed denoising-step
budget, isolating the speedup attributable to latent-cache reuse. 
\begin{algorithm}[t]
\caption{DiffusionGemma Entropy-Bound Sampling}
\label{alg:entropy-bound-sampler}

\KwIn{
Initial canvas $\rvz_T$;
maximum denoising steps $T$;
entropy bound $\eta$;
temperature limits $\tau_{\max},\tau_{\min}$;
embedding matrix $E$
}
\KwOut{Final denoised canvas $\widehat{\rvx}$}

$\mathbf{s}\leftarrow \mathbf{0}$
\tcp*[r]{initial self-conditioning signal}

\For{$n\leftarrow T$ \KwTo $1$}{

    $\mathbf{L}_n
    \leftarrow
    \operatorname{Decoder}
    (\rvz_n,\mathbf{s})$\;

    $\tau_n
    \leftarrow
    \tau_{\min}
    +
    (\tau_{\max}-\tau_{\min})\frac{n}{T}$\;

    $\mathbf{p}_n^l
    \leftarrow
    \operatorname{softmax}
    (\mathbf{L}_n^l/\tau_n)$
    for each position $l$\;

    $\widehat{x}_n^l
    \leftarrow
    \arg\max_{v\in\mathcal{V}}
    p_n^l(v)$\;

    $H_l
    \leftarrow
    -\sum_{v\in\mathcal{V}}
    p_n^l(v)\log p_n^l(v)$\;

    $\mathbf{s}
    \leftarrow
    \mathbf{p}_n E$
    \tcp*[r]{self-conditioning for next step}

    \If{$n>1$}{

        $\widetilde{x}_n^l
        \sim
        \operatorname{Cat}(\mathbf{p}_n^l)$\;

        $\pi
        \leftarrow
        \operatorname{argsort}(H)$
        \tcp*[r]{ascending entropy}

        $C_m
        \leftarrow
        \sum_{r=1}^{m}H_{\pi_r}$,
        \qquad
        $M_m
        \leftarrow
        \max_{1\leq r\leq m}H_{\pi_r}$\;

        $\mathcal{A}
        \leftarrow
        \left\{
        \pi_m:
        C_m-M_m\leq\eta
        \right\}$\;

        \For{$l\leftarrow1$ \KwTo $L$}{
            \eIf{$l\in\mathcal{A}$}{
                $\rvz_{n-1}^l
                \leftarrow
                \widetilde{x}_n^l$\;
            }{
                $\rvz_{n-1}^l
                \sim
                \operatorname{Unif}(\mathcal{V})$\;
            }
        }
    }
}

$\widehat{\rvx}
\leftarrow
(\widehat{x}_1^1,\ldots,\widehat{x}_1^L)$
\tcp*[r]{final argmax canvas}

\Return{$\widehat{\rvx}$}\;
\end{algorithm}

\textbf{Training Details.}
We train only the fusion module while keeping the pretrained model frozen. For the fusion module we use the low rank space r = 256. The gate bias is initialized to -3, but the up projection is 0 initialized so that at the start of training, we get the unchanged anchor output. Overall we train approximately 3.1M fusion parameters.

\textbf{SFT Data.}
We post-train on a 300K-example \emph{no-think} mixture
constructed from UltraData-SFT-2605~\citep{ultradata-sft-2605} and
Nemotron-Post-Training-Dataset-v2~\citep{nemotronpostv2}. The mixture contains
100K math, 100K code, 50K knowledge, 25K instruction-following, and 25K STEM
examples. We retain only samples marked as no-think or reasoning-off in their
source datasets. Each example is converted into eight diffusion canvases during
preprocessing.

The training configuration is:
\begin{itemize}[leftmargin=*,itemsep=0pt,topsep=2pt]
    \item \textbf{Training strategy:} Fusion training, DDP
    \item \textbf{Rollout cache age:} $k\in\{0,1,2\}$ with sampling
    probabilities $(0.34,0.33,0.33)$ and equal loss weights.
    \item \textbf{Model split:} 2 shared layers, 20 anchor layers, and
    8 denoiser layers.
    \item \textbf{Max response tokens:} 2048
    \item \textbf{Max Prompt length: }2048
    \item \textbf{Loss weights:} cache/decoder loss $1.0$, and KD loss $1.0$ with teacher temperature $1.0$.
    \item \textbf{Teacher:} frozen pretrained DiffusionGemma.
    \item \textbf{Self-conditioning:} probability $0.5$.
    \item \textbf{Time sampling:} antithetic time sampling enabled.
    \item \textbf{Gradient flow:} gradients from the denoiser to the encoder
    pathway are retained.
    \item \textbf{Optimization steps:} 18,750.
    \item \textbf{Batching:} per-device batch size $1$ with gradient
    accumulation over $4$ steps.
    \item \textbf{Optimizer:} AdamW with learning rate $1.5\times10^{-4}$,
    $\beta=(0.95,0.99)$, $\epsilon=10^{-8}$, and weight decay $10^{-4}$.
    \item \textbf{Learning-rate schedule:} cosine decay with 100 warmup steps
    and minimum learning-rate ratio $0.1$.
    \item \textbf{Precision:} BF16 mixed precision.
    \item \textbf{Gradient clipping:} maximum norm $1.0$.
    \item \textbf{Training seed:} $42$.
\end{itemize}

\textbf{Hardware.}
Training was performed with distributed data parallelism (DDP) on four
GB200 nodes, each with 189\,GB of GPU memory. Evaluation was performed on
a single GH200 node with 96\,GB of GPU memory.

\textbf{Benchmark Details.}
We evaluate TADM:Post-train on six reasoning, knowledge, and coding
benchmarks using the same decoding configuration for DiffusionGemma and
TADM. Since the post-training data contains no explicit reasoning
traces, all evaluations are performed in the model's \emph{no-think}
mode. We report mean accuracy and generation throughput over three random
seeds.

\begin{itemize}[leftmargin=*,itemsep=2pt,topsep=2pt]
    \item \textbf{GSM8K}~\citep{gsm8k}: grade-school mathematical
    reasoning; 5-shot prompting; maximum generation length 1,024 tokens;
    final-answer exact-match accuracy.

    \item \textbf{HumanEval}~\citep{humaneval}: functional Python code
    generation from natural-language specifications; 0-shot prompting;
    maximum generation length 1,024 tokens; execution-based pass@1.

    \item \textbf{GPQA-Diamond}~\citep{rein2024gpqa}: expert-level
    multiple-choice questions in biology, physics, and chemistry; 0-shot
    prompting; maximum generation length 2,048 tokens; multiple-choice
    accuracy.

    \item \textbf{MMLU-Pro}~\citep{wang2024mmlu}: challenging
    multiple-choice knowledge and reasoning benchmark with up to ten answer
    choices; 5-shot prompting; maximum generation length 2,048 tokens;
    multiple-choice accuracy.

    \item \textbf{AIME26}~\citep{aime26}: 30 problems from the 2026
    American Invitational Mathematics Examination; 0-shot prompting;
    maximum generation length 4,096 tokens; final-answer exact-match
    accuracy.

    \item \textbf{LiveCodeBench-v6}~\citep{lcb}: recent
    competitive-programming problems evaluated with execution-based
    correctness tests; 0-shot prompting; maximum generation length 2,000
    tokens; pass@1.
\end{itemize}

For throughput, we measure the number of generated tokens divided by
end-to-end generation time under the same hardware and decoding
configuration. DiffusionGemma evaluates all 30 text layers at every
reverse step. TADM:Post-train uses a $2/20/8$
shared/anchor/denoiser split with refresh interval $K=2$, so the
20-layer anchor network is evaluated once every two reverse steps while
the shared and denoiser networks are evaluated at every step. This
corresponds to $0.67\times$ the Transformer-layer evaluations of the
baseline, or a $33.3\%$ reduction.

The thing to note is that for the post-training experiment, we only train the fusion module which is very tiny in comparison to the entire network and we are able to achieve largely preserved task accuracy along with speedups at fixed denoising budget. Thus, demonstrating that we only require a small standalone module to achieve a modest speedup during inference.

\subsection{TADM:Pretrain}
\label{sec:app-tadm-pretrain}

\textbf{Architecture Details.}
TADM:Pretraining builds on the two-stage ADLM architecture
\citep{adlm}, replacing explicit anchor-token conditioning with the
time-anchored latent-cache formulation described in
Sec.~\ref{sec:stale-anchor}. We use Diffusion Transformer (DiT)
blocks for both the anchor and denoising networks. Following the ADLM
small-model configuration, each Transformer operates at hidden dimension
$d=768$ with 12 attention heads and a maximum sequence length of 1,024.
We use the GPT-2 tokenizer with vocabulary size $V=50{,}257$.

The model is decomposed as follows:
\begin{itemize}[leftmargin=*,itemsep=1pt,topsep=2pt]
    \item \textbf{Shared network:} the token embedding layer, evaluated for
    both the current state $\rvz_t$ and stale anchor state $\rvz_{t'}$.
    The embedding parameters are shared by the anchor and denoising pathways.

    \item \textbf{Anchor network:} 12 DiT blocks which transform
    $S(\rvz_{t'})$ into the cached deep representation
    $\mathbf{h}_{t'}$.

    \item \textbf{Denoiser network:} 6 DiT blocks which consume the fused
    representation and predict the clean-token distribution.

    \item \textbf{Fusion module:} a gated residual MLP that combines the
    shallow current representation $\mathbf{c}_t=S(\rvz_t)$ with the stale
    anchor $\mathbf{h}_{t'}$. No explicit time conditioning is used in the
    fusion module.
\end{itemize}

Specifically, defining
\[
\mathbf{s}_t=\operatorname{LN}(\mathbf{c}_t), \qquad
\mathbf{a}_{t'}=\operatorname{LN}(\mathbf{h}_{t'}), \qquad
\mathbf{u}_{t,t'}=[\mathbf{s}_t;\mathbf{a}_{t'}],
\]
the fusion module computes
\begin{align}
    \mathbf{g}_{t,t'}
    &=
    \sigma\!\left(W_g\mathbf{u}_{t,t'}+b_g\right), \\
    \boldsymbol{\Delta}_{t,t'}
    &=
    W_2\,\operatorname{GELU}
    \left(W_1\mathbf{u}_{t,t'}+b_1\right)+b_2, \\
    \widetilde{\mathbf{h}}_{t\mid t'}
    &=
    \operatorname{LN}\!\left(
        \mathbf{h}_{t'}
        +
        \mathbf{g}_{t,t'}\odot
        \boldsymbol{\Delta}_{t,t'}
    \right),
\end{align}
where
$W_1:\mathbb{R}^{2d}\rightarrow\mathbb{R}^{4d}$,
$W_2:\mathbb{R}^{4d}\rightarrow\mathbb{R}^{d}$, and
$\mathbf{g}_{t,t'}\in[0,1]^d$ is a per-token, per-hidden-dimension gate.
The final update projection $W_2$ and bias $b_2$ are initialized to zero,
so the fusion module initially uses stale anchor as the primary input to the denoiser without any added correction.
Adding the output layer norm is an implementation choice.

\textbf{Dataset.}
We pretrain TADM on OpenWebText (OWT)~\citep{owt}, using the same
1,024-token sequence length and GPT-2 tokenization setup as ADLM.
Training is performed for 1M optimization steps, corresponding to the
same training-step budget used for the 1M-step ADLM comparison.

\textbf{Training Details.}
We train the full TADM:Pretraining model from scratch using absorbing-state
masked diffusion with a log-linear noise schedule. The current diffusion
time $t$ is sampled continuously, while the stale anchor time $t'$ is
constructed by sampling a discrete cache age aligned with the reverse
sampling budgets used at inference. Time conditioning is disabled. We set $\sigma_t=0$, the remasking probability in the reverse posterior during training. Hence, the forward corruption from $z_t$ to $z_{t'}$ is valid. Experimentally, we sample $t$ continuously as $t\sim\mathcal{U}(\epsilon,1)$. We independently sample an anchor refresh interval $K\in\{1,2,4,8\}$ and a sampling-step budget $T\in\{128,256,512,1024,2048,4096\}$. Conditioned on $K$, the cache age is sampled as $k\sim\mathcal{U}\{0,\ldots,K-1\}$, and
$t'=\min\left(1,t+\frac{k}{T}\right)$.
Thus $t$ is continuous, while the offset between $t$ and $t'$ is aligned with the discrete cache ages encountered at inference time. 

The training configuration is:
\begin{itemize}[leftmargin=*,itemsep=1pt,topsep=2pt]
    \item \textbf{Optimization steps:} 1,000,000.
    \item \textbf{Global batch size:} 512.
    \item \textbf{Sequence length:} 1,024 tokens.
    \item \textbf{Hardware:} 32 NVIDIA GH200 GPUs using DDP.
    \item \textbf{Precision:} BF16.
    \item \textbf{Optimizer:} AdamW with learning rate
    $3\times10^{-4}$, $\beta_1=0.9$, $\beta_2=0.999$,
    $\epsilon=10^{-8}$, and zero weight decay.
    \item \textbf{Gradient clipping:} maximum norm $1.0$.
    \item \textbf{EMA:} $0.9999$.
    \item \textbf{Noise schedule:} log-linear absorbing-state diffusion.
    \item \textbf{Time conditioning:} disabled.
    \item \textbf{Time sampling:} continuous $t$ with antithetic sampling.
    \item \textbf{Sampling-step budgets used to construct cache ages:}
    $T\in\{128,256,512,1024,2048,4096\}$.
    \item \textbf{Anchor refresh intervals during training:}
    sampled across multiple cache ages so that the model observes both
    fresh and stale anchors.
\end{itemize}

\textbf{Sampler Details.}
We use the ReMDM remasking sampler~\citep{remdm}, following the released ReMDM sampling configuration, along with the anchor cache inference (Alg \ref{alg:anchor-cache-inference}). Although, we set $\sigma_t=0$ during training, we still use remasking samplers as its already established in ReMDM~\citep{remdm} that the same weights can be used for remasking samplers as well. Let $L=1024$ denote the generation
length. For sampling budgets $T<L$, we use the ReMDM-Cap sampler,
whereas for $T\geq L$ we use ReMDM-Loop. We evaluate
$T\in\{128,256,512,1024,2048,4096\}$ and vary the anchor refresh
interval $K$ to control the amount of latent-cache reuse. 

The sampling configuration is:
\begin{itemize}[leftmargin=*,itemsep=1pt,topsep=2pt]
    \item \textbf{Generation length:} $L=1024$ tokens.
    \item \textbf{Sampling budgets:}
    $T\in\{128,256,512,1024,2048,4096\}$.
    \item \textbf{Timestep schedule:} linear.
    \item \textbf{Nucleus sampling:} $p=0.9$.
    \item \textbf{ReMDM-Loop ($T\geq1024$):}
    $\eta=0.02$, $t_{\mathrm{on}}=0.55$,
    $t_{\mathrm{off}}=0.05$, and $\alpha_{\mathrm{on}}=0.9$.
    \item \textbf{ReMDM-Cap ($T<1024$):}
    $\eta=0.04$, with the remaining remasking parameters unchanged.
    \item \textbf{Anchor refresh intervals:}
    $K\in\{1,2,4,8\}$ for the reported refresh-interval sweep, where
    $K=1$ recomputes the anchor at every reverse step.
    \item \textbf{Batch Size:} 1
\end{itemize}
\textbf{Hardware.}
Training was performed with distributed data parallelism (DDP) on 32
GH200 nodes, each with 96\,GB of GPU memory. Evaluation was performed on
a single GH200 node with 96\,GB of GPU memory.

\textbf{Evaluation.}
We evaluate unconditional 1,024-token generation on OpenWebText over
multiple reverse-step budgets and anchor refresh intervals. For each
configuration, we generate 5,000 samples and evaluate generation quality
using MAUVE, generative perplexity (Gen PPL), and token entropy. Gen PPL
and MAUVE use GPT-2 Large as the reference evaluator. We additionally
report Transformer-layer evaluations and measured generation throughput
to quantify the quality--compute trade-off induced by latent-cache reuse.

\begin{itemize}[leftmargin=*,itemsep=1pt,topsep=2pt]
    \item \textbf{Generation setting:} unconditional generation.
    \item \textbf{Generated sequence length:} 1,024 tokens.
    \item \textbf{Quality evaluation:} 5,000 generated samples per
    configuration.
    \item \textbf{Metrics:} MAUVE, Gen PPL, entropy, Transformer-layer
    compute, and tokens/s.
    \item \textbf{Gen PPL/MAUVE evaluator:} GPT-2 Large.
    \item \textbf{Sampling budgets:}
    $T\in\{128,256,512,1024,2048,4096\}$.
    \item \textbf{Cache sweep:} $K\in\{1,2,4,8\}$.
\end{itemize}

\textbf{Token Budget: } We train our diffusion language model upto 1M steps and therefore as per the token budget calculation for a standard masked Diffusion Language model as stated by MDLM~\citep{mdlm}, the token budget is 0.5 * 1M * 1024 * 512 = 262B tokens.

\section{Additional Experiments}
\label{sec:additional-experiments}

We provide additional experiments examining four aspects of time-based
anchoring: (i) zero-shot likelihood generalization of the pretrained
models, (ii) the effect of anchor refresh interval on the
quality--throughput trade-off, (iii) whether an ADLM trained without
temporal reuse can tolerate stale anchors, and (iv) the corresponding
quality--compute trade-off for TADM:Post-train. We additionally provide
qualitative generations illustrating model behavior at different cache
ages.

\textbf{Zero-Shot Likelihood Evaluation.} Table~\ref{tab:zero-shot-ppl} evaluates the zero-shot likelihood
generalization of TADM:Pretraining on datasets that are unseen during
training. Although TADM is slightly behind ADLM on several benchmarks,
its perplexity remains competitive with MDLM and the other diffusion
baselines across domains. In particular, TADM matches or improves upon
MDLM on several datasets, including Lambada, LM1B, AG News, PubMed, and
ArXiv, while remaining comparable on Wikitext and PTB.

These results indicate that training the model to operate with temporally
stale latent representations does not restrict the learned representation
to the OpenWebText training distribution or to a particular sampling
trajectory. Instead, the latent structure learned under temporal reuse
continues to support likelihood modeling on unseen text domains. This
suggests that the cacheable representations learned by TADM capture
features that generalize across datasets, rather than merely memorizing
dataset-specific correlations required for stale-anchor reuse.
\begin{table*}
\vspace{-0.5ex}
\small
\caption{Zero-shot validation perplexities ($\downarrow$) on OWT-trained models with 1,024 NFEs. TADM:Pretraining $\dagger$ is the $\gamma=0$ variant and TADM:Pretraining$\ddagger$ is $\gamma=3e-3$.}
\vspace{-1ex}
\label{tab:zero-shot-ppl}
\setlength{\tabcolsep}{4.3pt}
\resizebox{\textwidth}{!}{
\begin{tabular}{lccccccc}
\toprule
Model & Lambada & PTB & Wikitext & LM1B & AG News & PubMed & ArXiv \\
\midrule
AR & 51.28 & 82.05 & 25.75 & 51.25 & 52.09 & 49.01 & 41.73 \\
\midrule
\textit{AR+Diffusion} & & & & & & & \\
\rowcolor{gray!10}
BD3-LM ($L'=4$) & 50.03 & 96.81 & 31.31 & 60.88 & 61.67 & 42.52 & 39.20 \\
\midrule
\textit{Diffusion} & & & & & & & \\
SEDD & 49.86 & 100.09 & 34.28 & 68.20 & 62.09 & 44.53 & 38.38 \\
MDLM & 47.52 & 95.26 & 32.83 & 67.01 & 61.15 & 41.89 & 37.37 \\
ADLM (262B) & 44.93 & 98.16 & 32.45 & 65.59 & 57.10 & 38.29 & 35.08 \\
ADLM (524B) & \textbf{44.32} & \textbf{95.37} & \textbf{31.94} & 64.43 & \textbf{55.72} & \textbf{37.56} & \textbf{33.69} \\
\midrule
\rowcolor{orange!25}
TADM:Pretraining $\dagger$ (262B) & 46.04 & 106.47 & 35.03 & 65.37 & 59.90 & 40.08 & 35.08 \\
\rowcolor{orange!25}
TADM:Pretraining $\ddagger$ (262B) & 45.48 & 100.92 & 33.77 & \textbf{64.27} & 59.09 & 40.32 & 35.48 \\
\bottomrule
\end{tabular}
}
\vspace{-2ex}
\end{table*}

\textbf{Effect of Anchor Refresh Interval.}
Table~\ref{tab:TADM-refresh-interval-sweep} isolates the effect of cache
staleness by sweeping $K\in\{1,2,4,8\}$ while holding the sampling budget
$T$ fixed. Increasing $K$ monotonically improves throughput because the
12-layer anchor network is evaluated less frequently. At small sampling
budgets ($T=128$ and $T=256$), this comes with a noticeable degradation in
generation quality, reflected by lower MAUVE and higher Gen PPL. In
contrast, at larger $T$, Gen PPL remains nearly unchanged across cache
ages, indicating substantially greater tolerance to stale anchors. For
example, at $T=2048$, TADM:Pretraining$^{\dagger}$ changes only from Gen
PPL $21.49$ at $K=1$ to $22.34$ at $K=8$, while throughput increases from
$26.59$ to $41.19$ Tok/s. At $T=4096$, Gen PPL remains between $17.21$
and $17.36$ across the full sweep.

The two TADM variants exhibit very similar trends across cache ages.
Thus, the ability to reuse stale latent representations is primarily
learned through the temporal denoising objective itself, while the optional
anchor supervision has only a modest effect on generation quality in this
ablation.

\begin{table*}[!t]
\centering
\scriptsize
\caption{
Ablation over the anchor refresh interval $K$ for the 1M-step
TADM:Pretraining$^{\dagger}$ checkpoint and the $x_{t'}$-masked
TADM:Pretraining$^{\ddagger}$ checkpoint. Quality metrics use 5,000 OWT samples
per setting, while throughput is measured over 20 samples of length
1,024. Increasing $K$ reuses the cached anchor for more reverse steps.
}
\label{tab:TADM-refresh-interval-sweep}

\textbf{$T=128$ and $T=256$}\\[0.4ex]
\resizebox{\textwidth}{!}{%
\begin{tabular}{lcccccccc}
\toprule
Method ($K$)
& \multicolumn{4}{c}{$T=128$}
& \multicolumn{4}{c}{$T=256$} \\
\cmidrule(lr){2-5}\cmidrule(lr){6-9}
& MAUVE $\uparrow$ & Gen PPL $\downarrow$ & Entropy $\uparrow$ & Tok/s $\uparrow$
& MAUVE $\uparrow$ & Gen PPL $\downarrow$ & Entropy $\uparrow$ & Tok/s $\uparrow$ \\
\midrule
TADM:Pretraining$^{\dagger}$ ($K=1$) & 0.135 & 38.36 & 5.425 & 347.69 & 0.304 & 28.10 & 5.334 & 175.25 \\
TADM:Pretraining$^{\dagger}$ ($K=2$) & 0.084 & 41.56 & 5.433 & 445.23 & 0.239 & 29.86 & 5.339 & 222.28 \\
TADM:Pretraining$^{\dagger}$ ($K=4$) & 0.040 & 47.12 & 5.447 & 512.49 & 0.132 & 32.53 & 5.348 & 257.61 \\
TADM:Pretraining$^{\dagger}$ ($K=8$) & 0.016 & 58.33 & 5.465 & 554.59 & 0.052 & 38.22 & 5.364 & 279.42 \\
\midrule
TADM:Pretraining$^{\ddagger}$ ($K=1$) & 0.108 & 39.75 & 5.432 & 348.58 & 0.341 & 29.06 & 5.347 & 176.12 \\
TADM:Pretraining$^{\ddagger}$ ($K=2$) & 0.070 & 43.06 & 5.433 & 447.20 & 0.220 & 30.59 & 5.337 & 223.68 \\
TADM:Pretraining$^{\ddagger}$ ($K=4$) & 0.043 & 48.29 & 5.437 & 515.46 & 0.160 & 33.25 & 5.344 & 258.88 \\
TADM:Pretraining$^{\ddagger}$ ($K=8$) & 0.017 & 58.48 & 5.451 & 560.19 & 0.084 & 37.70 & 5.338 & 280.44 \\
\bottomrule
\end{tabular}
}

\vspace{0.4ex}

\textbf{$T=512$ and $T=1024$}\\[0.4ex]
\resizebox{\textwidth}{!}{%
\begin{tabular}{lcccccccc}
\toprule
Method ($K$)
& \multicolumn{4}{c}{$T=512$}
& \multicolumn{4}{c}{$T=1024$} \\
\cmidrule(lr){2-5}\cmidrule(lr){6-9}
& MAUVE $\uparrow$ & Gen PPL $\downarrow$ & Entropy $\uparrow$ & Tok/s $\uparrow$
& MAUVE $\uparrow$ & Gen PPL $\downarrow$ & Entropy $\uparrow$ & Tok/s $\uparrow$ \\
\midrule
TADM:Pretraining$^{\dagger}$ ($K=1$) & 0.386 & 20.61 & 5.232 & 88.11 & 0.564 & 26.15 & 5.371 & 47.02 \\
TADM:Pretraining$^{\dagger}$ ($K=2$) & 0.339 & 21.38 & 5.230 & 111.95 & 0.570 & 26.49 & 5.373 & 58.68 \\
TADM:Pretraining$^{\dagger}$ ($K=4$) & 0.285 & 22.60 & 5.233 & 129.54 & 0.440 & 27.38 & 5.380 & 68.39 \\
TADM:Pretraining$^{\dagger}$ ($K=8$) & 0.181 & 24.79 & 5.230 & 140.60 & 0.405 & 28.65 & 5.383 & 74.54 \\
\midrule
TADM:Pretraining$^{\ddagger}$ ($K=1$) & 0.450 & 20.91 & 5.240 & 88.10 & 0.596 & 26.57 & 5.379 & 46.82 \\
TADM:Pretraining$^{\ddagger}$ ($K=2$) & 0.363 & 21.74 & 5.233 & 111.50 & 0.546 & 26.82 & 5.373 & 58.39 \\
TADM:Pretraining$^{\ddagger}$ ($K=4$) & 0.292 & 22.66 & 5.218 & 129.19 & 0.510 & 27.35 & 5.359 & 67.95 \\
TADM:Pretraining$^{\ddagger}$ ($K=8$) & 0.225 & 24.17 & 5.183 & 140.35 & 0.441 & 28.41 & 5.344 & 74.34 \\
\bottomrule
\end{tabular}
}

\vspace{0.4ex}

\textbf{$T=2048$ and $T=4096$}\\[0.4ex]
\resizebox{\textwidth}{!}{%
\begin{tabular}{lcccccccc}
\toprule
Method ($K$)
& \multicolumn{4}{c}{$T=2048$}
& \multicolumn{4}{c}{$T=4096$} \\
\cmidrule(lr){2-5}\cmidrule(lr){6-9}
& MAUVE $\uparrow$ & Gen PPL $\downarrow$ & Entropy $\uparrow$ & Tok/s $\uparrow$
& MAUVE $\uparrow$ & Gen PPL $\downarrow$ & Entropy $\uparrow$ & Tok/s $\uparrow$ \\
\midrule
TADM:Pretraining$^{\dagger}$ ($K=1$) & 0.644 & 21.49 & 5.313 & 26.59 & 0.685 & 17.21 & 5.232 & 15.28 \\
TADM:Pretraining$^{\dagger}$ ($K=2$) & 0.626 & 21.56 & 5.315 & 32.29 & 0.659 & 17.28 & 5.237 & 18.32 \\
TADM:Pretraining$^{\dagger}$ ($K=4$) & 0.650 & 21.77 & 5.314 & 37.42 & 0.618 & 17.22 & 5.226 & 20.94 \\
TADM:Pretraining$^{\dagger}$ ($K=8$) & 0.557 & 22.34 & 5.310 & 41.19 & 0.602 & 17.36 & 5.214 & 22.99 \\
\midrule
TADM:Pretraining$^{\ddagger}$ ($K=1$) & 0.669 & 21.58 & 5.319 & 26.45 & 0.689 & 17.29 & 5.233 & 15.22 \\
TADM:Pretraining$^{\ddagger}$ ($K=2$) & 0.637 & 21.73 & 5.309 & 32.23 & 0.696 & 17.26 & 5.231 & 18.28 \\
TADM:Pretraining$^{\ddagger}$ ($K=4$) & 0.646 & 21.70 & 5.293 & 37.35 & 0.637 & 17.19 & 5.221 & 20.94 \\
TADM:Pretraining$^{\ddagger}$ ($K=8$) & 0.535 & 22.27 & 5.287 & 41.18 & 0.639 & 17.21 & 5.203 & 23.03 \\
\bottomrule
\end{tabular}
}
\end{table*}

\paragraph{Stale-anchor reuse in ADLM.}
Table~\ref{tab:adlm-tadm-cache-reuse} compares stale-anchor reuse in ADLM
against TADM:Pretraining, with each configuration averaged over 20
generations to provide a direct diagnostic of generation quality and
throughput as the anchor refresh interval $K$ increases. ADLM is trained
only with fresh anchors and therefore tests whether latent representations
can be reused without explicitly learning temporal reuse. As $K$ increases,
ADLM degrades rapidly, whereas TADM remains substantially more stable. For
example, at $T=1024$, increasing $K$ from $1$ to $8$ raises ADLM Gen PPL
from $25.97$ to $144.97$, while TADM:Pretraining$^{\dagger}$ changes only
from $23.18$ to $30.52$ and increases throughput from $47.02$ to
$74.54$ Tok/s.

This comparison also clarifies the complementary roles of stale-anchor
training and fusion. Stale-anchor training exposes the model to the
temporal mismatch created by reusing an anchor from an earlier diffusion
state; without it, as illustrated by ADLM, increasing anchor age rapidly
degrades generation. The fusion module addresses the complementary problem:
the stale anchor does not contain the latest information in the evolving
canvas, so fusion incorporates the current-state representation before
denoising. Thus, stale-anchor training teaches the model to operate under
temporal reuse, while fusion provides the mechanism for correcting the
cached representation using current information. Without stale-anchor
training the model is not optimized for reuse, while without fusion the
current canvas cannot effectively update the stale anchor. Together, these
components enable stable latent-cache reuse across multiple reverse steps.

\begin{table*}[t]
\centering
\scriptsize
\setlength{\tabcolsep}{2.6pt}
\renewcommand{\arraystretch}{0.90}

\caption{
\textbf{Effect of stale-anchor reuse in ADLM and TADM:Pretraining.}
We compare anchor refresh intervals $K\in\{1,2,4,8\}$, where $K=1$
recomputes the anchor at every reverse step. We report throughput
(Tok/s), generative perplexity (Gen PPL), and token entropy.
}
\label{tab:adlm-tadm-cache-reuse}

\resizebox{\textwidth}{!}{%
\begin{tabular}{@{}l c rrr rrr rrr@{}}
\toprule
Method & $K$
& \multicolumn{3}{c}{$T=128$}
& \multicolumn{3}{c}{$T=256$}
& \multicolumn{3}{c}{$T=512$} \\
\cmidrule(lr){3-5}
\cmidrule(lr){6-8}
\cmidrule(lr){9-11}
& &
Tok/s $\uparrow$ & Gen PPL $\downarrow$ & Ent.
&
Tok/s $\uparrow$ & Gen PPL $\downarrow$ & Ent.
&
Tok/s $\uparrow$ & Gen PPL $\downarrow$ & Ent. \\
\midrule

\multirow{4}{*}{ADLM}
& 1 & 320.62 & 57.29  & 5.51
    & 161.31 & 41.88  & 5.46
    & 81.44  & 31.63  & 5.33 \\
& 2 & 404.51 & 78.98  & 5.59
    & 202.45 & 60.78  & 5.52
    & 102.17 & 48.01  & 5.46 \\
& 4 & 462.16 & 194.00 & 5.66
    & 231.85 & 106.46 & 5.58
    & 117.30 & 76.38  & 5.44 \\
& 8 & 498.35 & 680.71 & 5.86
    & 250.77 & 245.08 & 5.65
    & 126.79 & 150.22 & 5.47 \\

\midrule

\multirow{4}{*}{TADM$^{\dagger}$}
& 1 & 347.69 & 42.25 & 5.448
    & 175.25 & 27.97 & 5.347
    & 88.11  & 20.46 & 5.252 \\
& 2 & 445.23 & 42.12 & 5.425
    & 222.28 & 30.83 & 5.385
    & 111.95 & 21.47 & 5.275 \\
& 4 & 512.49 & 46.57 & 5.463
    & 257.61 & 33.44 & 5.375
    & 129.54 & 20.74 & 5.155 \\
& 8 & 554.59 & 50.00 & 5.424
    & 279.42 & 40.55 & 5.364
    & 140.60 & 27.05 & 5.263 \\

\midrule

\multirow{4}{*}{TADM$^{\ddagger}$}
& 1 & 348.58 & 40.74 & 5.448
    & 176.12 & 28.81 & 5.348
    & 88.10  & 20.59 & 5.247 \\
& 2 & 447.20 & 40.42 & 5.414
    & 223.68 & 29.94 & 5.329
    & 111.50 & 24.89 & 5.291 \\
& 4 & 515.46 & 53.31 & 5.448
    & 258.88 & 35.14 & 5.370
    & 129.19 & 21.76 & 5.213 \\
& 8 & 560.19 & 49.04 & 5.414
    & 280.44 & 45.09 & 5.397
    & 140.35 & 26.37 & 5.249 \\

\bottomrule
\end{tabular}%
}

\vspace{1.0ex}

\resizebox{\textwidth}{!}{%
\begin{tabular}{@{}l c rrr rrr rrr@{}}
\toprule
Method & $K$
& \multicolumn{3}{c}{$T=1024$}
& \multicolumn{3}{c}{$T=2048$}
& \multicolumn{3}{c}{$T=4096$} \\
\cmidrule(lr){3-5}
\cmidrule(lr){6-8}
\cmidrule(lr){9-11}
& &
Tok/s $\uparrow$ & Gen PPL $\downarrow$ & Ent.
&
Tok/s $\uparrow$ & Gen PPL $\downarrow$ & Ent.
&
Tok/s $\uparrow$ & Gen PPL $\downarrow$ & Ent. \\
\midrule

\multirow{4}{*}{ADLM}
& 1 & 43.19 & 25.97  & 5.25
    & 24.32 & 17.82  & 5.11
    & 14.01 & 13.54  & 5.15 \\
& 2 & 53.25 & 38.08  & 5.37
    & 29.36 & 32.71  & 5.24
    & 16.70 & 26.03  & 4.98 \\
& 4 & 61.43 & 70.99  & 5.31
    & 33.75 & 60.95  & 5.13
    & 18.92 & 40.08  & 4.73 \\
& 8 & 66.59 & 144.97 & 5.29
    & 36.78 & 111.56 & 5.09
    & 20.62 & 87.37  & 4.75 \\

\midrule

\multirow{4}{*}{TADM$^{\dagger}$}
& 1 & 47.02 & 23.18 & 5.320
    & 26.59 & 20.07 & 5.259
    & 15.28 & 17.86 & 5.229 \\
& 2 & 58.68 & 28.06 & 5.338
    & 32.29 & 20.96 & 5.285
    & 18.32 & 17.24 & 5.324 \\
& 4 & 68.39 & 29.43 & 5.416
    & 37.42 & 25.35 & 5.377
    & 20.94 & 16.20 & 5.180 \\
& 8 & 74.54 & 30.52 & 5.423
    & 41.19 & 24.47 & 5.333
    & 22.99 & 17.01 & 5.228 \\

\midrule

\multirow{4}{*}{TADM$^{\ddagger}$}
& 1 & 46.82 & 25.48 & 5.362
    & 26.45 & 22.13 & 5.345
    & 15.22 & 17.15 & 5.257 \\
& 2 & 58.39 & 27.37 & 5.332
    & 32.23 & 19.63 & 5.235
    & 18.28 & 17.84 & 5.300 \\
& 4 & 67.95 & 25.61 & 5.275
    & 37.35 & 21.99 & 5.367
    & 20.94 & 16.44 & 5.094 \\
& 8 & 74.34 & 31.45 & 5.388
    & 41.18 & 24.37 & 5.361
    & 23.03 & 17.14 & 5.202 \\

\bottomrule
\end{tabular}%
}

\end{table*}

\textbf{Stale-anchor reuse in DiffusionGemma.}
Table~\ref{tab:naive-vs-tadm-posttrain} compares naive stale-anchor reuse in
DiffusionGemma with TADM:Post-train. For the baseline, we follow
Algorithm~\ref{alg:anchor-cache-inference} but disable fusion, so that the
denoiser directly receives the stale anchor representation for the next
$K$ steps. Corruption is measured using Qwen3-30B-A3B-Instruct-2507
\citep{qwen30b} as a judge for incoherence or repetition, with results
averaged over seeds $\{0,1,2\}$. On relatively easier benchmarks such as
GSM8K and HumanEval, both accuracy and corruption remain largely stable
under increasing reuse. In contrast, the harder reasoning and coding
benchmarks are substantially more sensitive to stale representations:
at $K=3$, naive reuse increases corruption to $37.78\%$ on AIME26,
$14.81\%$ on GPQA-Diamond, and $32.00\%$ on LiveCodeBench-v6, while also
reducing task accuracy. TADM:Post-train substantially mitigates this
degradation, reducing the corresponding corruption rates to $20.00\%$,
$4.71\%$, and $21.71\%$ while preserving substantially more of the
original task performance. Nevertheless, corruption at $K=3$ remains
higher than the fresh-anchor baseline on some difficult tasks, exposing a
limitation of increasingly stale reuse. Overall, these results show that
naively caching deep representations is insufficient for challenging
generation tasks and support our combined use of stale-anchor training and
current-state fusion to make temporal latent reuse more robust.

\begin{table*}[t]
\centering
\setlength{\tabcolsep}{3.2pt}
\renewcommand{\arraystretch}{0.95}

\caption{
\textbf{Naive stale-anchor reuse versus TADM:Post-train.}
Results are averaged over seeds $\{0,1,2\}$ with adaptive stopping disabled.
The original cache-age notation $k=\{0,1,2\}$ is reported here as anchor
refresh intervals $K=\{1,2,3\}$. ``Baseline'' denotes naive reuse with the
original DiffusionGemma weights and no fusion modules, while ``Ours'' denotes
TADM:Post-train. Speed is reported as tokens/s, with speedup in parentheses relative to the
corresponding baseline $K=1$ throughput. Corruption is the percentage of
generations flagged as incoherent or containing a repetition loop. 
}
\label{tab:naive-vs-tadm-posttrain}

\begin{tabular}{lc
                rr
                rr
                rr}
\toprule
& &
\multicolumn{2}{c}{Accuracy (\%)}
& \multicolumn{2}{c}{Corruption (\%)}
& \multicolumn{2}{c}{Tok./s (speedup)} \\
\cmidrule(lr){3-4}
\cmidrule(lr){5-6}
\cmidrule(lr){7-8}
Benchmark & $K$
& Baseline & Ours
& Baseline & Ours
& Baseline & Ours \\
\midrule

\multirow{3}{*}{AIME26}
& 1 & 47.78 & 47.78 & 11.11 & 11.11 & 62.79 (1.00$\times$) & 57.48 (0.92$\times$) \\
& 2 & 44.44 & 46.67 & 15.56 & 14.44 & 90.67 (1.44$\times$) & 83.41 (1.33$\times$) \\
& 3 & 44.44 & 48.89 & 37.78 & 20.00 & 105.67 (1.68$\times$) & 93.50 (1.49$\times$) \\

\midrule

\multirow{3}{*}{GSM8K}
& 1 & 94.79 & 94.79 & 0.08 & 0.08 & 33.80 (1.00$\times$) & 35.72 (1.06$\times$) \\
& 2 & 94.69 & 94.79 & 0.13 & 0.05 & 51.82 (1.53$\times$) & 51.51 (1.52$\times$) \\
& 3 & 94.72 & 94.95 & 0.15 & 0.23 & 56.68 (1.68$\times$) & 59.53 (1.76$\times$) \\

\midrule

\multirow{3}{*}{HumanEval}
& 1 & 95.12 & 95.12 & 4.88 & 4.88 & 41.84 (1.00$\times$) & 37.96 (0.91$\times$) \\
& 2 & 92.48 & 95.53 & 5.28 & 4.88 & 62.04 (1.48$\times$) & 58.22 (1.39$\times$) \\
& 3 & 93.70 & 94.72 & 4.67 & 3.46 & 72.27 (1.73$\times$) & 63.85 (1.53$\times$) \\

\midrule

\multirow{3}{*}{GPQA-D}
& 1 & 67.00 & 67.00 & 2.36 & 2.36 & 60.31 (1.00$\times$) & 55.30 (0.92$\times$) \\
& 2 & 64.81 & 66.16 & 4.71 & 1.85 & 80.44 (1.33$\times$) & 79.89 (1.32$\times$) \\
& 3 & 61.45 & 66.16 & 14.81 & 4.71 & 95.02 (1.58$\times$) & 93.90 (1.56$\times$) \\

\midrule

\multirow{3}{*}{LCB-v6}
& 1 & 50.29 & 50.29 & 13.90 & 13.90 & 58.35 (1.00$\times$) & 56.65 (0.97$\times$) \\
& 2 & 49.90 & 52.38 & 23.62 & 16.95 & 85.91 (1.47$\times$) & 78.66 (1.35$\times$) \\
& 3 & 47.43 & 50.10 & 32.00 & 21.71 & 104.06 (1.78$\times$) & 104.44 (1.79$\times$) \\

\bottomrule
\end{tabular}%

\end{table*}

\tcbset{
promptbox/.style={
    enhanced,
    breakable,
    colback=red!5,
    colbacktitle=red!12,
    colframe=red!55!black,
    boxrule=0.7pt,
    arc=2pt,
    left=5pt,
    right=5pt,
    top=4pt,
    bottom=4pt,
    fonttitle=\bfseries\small,
    coltitle=black
},
    responsebox/.style={
        enhanced,
        breakable,
        colback=yellow!10,
        colframe=yellow!55!black,
        boxrule=0.7pt,
        arc=2pt,
        left=5pt,
        right=5pt,
        top=4pt,
        bottom=4pt,
        fonttitle=\bfseries\small,
        coltitle=black
    }
}
\newpage
\subsection{Qualitative TADM:Post-train Generations}
The following section shows 6 sample generations from TADM:Post-train on AIME 2026 and GPQA Diamond (3 each) for anchor refresh intervals $K\in\{1,2,3\}$. Each benchmark uses the same prompt across the three anchor refresh intervals. Generation $i$ --- $K=k$ means the ith generation uses K=k, it is for numbering purpose only. The first three generations 1-3 correspond to the AIME 2026 for anchor refresh K=1-3 respectively, whereas the last three generations, 4-6, are for GPQA-Diamond for anchor refresh K $1 -3$ respectively.

\label{sec:qualitative-posttrain}

\subsubsection{AIME-2026}

\noindent
\textbf{Sample ID: 0}
\hfill
\textbf{Reference answer: 277}

\begin{tcolorbox}[
    promptbox,
    title={Shared prompt}
]
\small

Solve the following problem. Put the final integer answer in
$\boxed{\ldots}$.

\medskip

Question: Patrick started walking at a constant rate along a straight
road from school to the park. One hour after Patrick left, Tanya
started running along the same road from school to the park. One hour
after Tanya left, Jose started bicycling along the same road from school
to the park. Tanya ran at a constant rate of $2$ miles per hour faster
than Patrick walked, Jose bicycled at a constant rate of $7$ miles per
hour faster than Tanya ran, and all three arrived at the park at the
same time. The distance from the school to the park is $\frac{m}{n}$
miles, where $m$ and $n$ are relatively prime positive integers. Find
$m+n$.

Answer:
\end{tcolorbox}

\begin{tcolorbox}[
    responsebox,
    title={Generation 1 --- $K=1$}
]
\small

Let $v$ be the speed at which Patrick walks in miles per hour.

Let $d$ be the distance from the school to the park in miles.

According to the problem:

\begin{itemize}[leftmargin=2em,itemsep=0pt]
    \item Patrick's speed is $v$.
    \item Tanya's speed is $v + 2$.
    \item Jose's speed is $(v + 2) + 7 = v + 9$.
\end{itemize}

Let $t$ be the time (in hours) that Patrick took to travel from school
to the park.

Since Tanya started one hour after Patrick, her travel time is $t - 1$.

Since Jose started one hour after Tanya, his travel time is $t - 2$.

Because all three traveled the same distance $d$, we can set up the
following equations based on the formula
$d = \text{speed} \times \text{time}$:

\begin{enumerate}[leftmargin=2em,itemsep=0pt]
    \item $d = v \cdot t$
    \item $d = (v + 2)(t - 1)$
    \item $d = (v + 9)(t - 2)$
\end{enumerate}

We set the expressions for distance equal to each other:

\[
vt = (v + 2)(t - 1)
\]

\[
vt = vt - v + 2t - 2
\]

\[
0 = -v + 2t - 2
\implies
v = 2t - 2
\quad \text{--- (Eq. 1)}
\]

Next, we set the first and third expressions equal:

\[
vt = (v + 9)(t - 2)
\]

\[
vt = vt - 2v + 9t - 18
\]

\[
0 = -2v + 9t - 18
\implies
2v = 9t - 18
\quad \text{--- (Eq. 2)}
\]

Now, substitute Eq. 1 into Eq. 2:

\[
2(2t - 2) = 9t - 18
\]

\[
4t - 4 = 9t - 18
\]

\[
14 = 5t
\]

\[
t = \frac{14}{5}
\]

Now find the value of $v$ using Eq. 1:

\[
v = 2\left(\frac{14}{5}\right) - 2
= \frac{28}{5} - \frac{10}{5}
= \frac{18}{5}
\]

Finally, calculate the distance $d$:

\[
d = v \cdot t
= \left(\frac{18}{5}\right) \cdot
  \left(\frac{14}{5}\right)
= \frac{252}{25}
\]

We check if $m = 252$ and $n = 25$ are relatively prime.

The prime factors of 25 are just 5. 252 is not divisible by 5
(it does not end in 0 or 5).

Thus, $\gcd(252, 25) = 1$.

The problem asks for $m + n$:

\[
252 + 25 = 277
\]

\[
\boxed{277}
\]

\end{tcolorbox}

\begin{tcolorbox}[
    responsebox,
    title={Generation 2 --- $K=2$}
]
\small

Let $v$ be the rate at which Patrick walks in miles per hour.

Let $d$ be the distance from the school to the park in miles.

According to the problem:

\begin{enumerate}[leftmargin=2em,itemsep=0pt]
    \item Patrick's rate is $v$.
    \item Tanya's rate is $v + 2$.
    \item Jose's rate is $(v + 2) + 7 = v + 9$.
\end{enumerate}

Let $t$ be the time (in hours) that it took Patrick to travel from the
school to the park.

Since Tanya started one hour after Patrick, her travel time was $t - 1$.

Since Jose started one hour after Tanya, his travel time was $t - 2$.

Because all three arrived at the park at the same time and traveled the
same distance $d$, we can write the following equations:

\[
d = v \cdot t
\]

\[
d = (v + 2)(t - 1)
\]

\[
d = (v + 9)(t - 2)
\]

We set the expressions for $d$ equal to each other:

\[
vt = (v + 2)(t - 1)
\]

\[
vt = vt - v + 2t - 2
\]

\[
0 = -v + 2t - 2
\implies
v = 2t - 2
\quad \text{--- (Equation 1)}
\]

Now, set the first and third expressions equal:

\[
vt = (v + 9)(t - 2)
\]

\[
vt = vt - 2v + 9t - 18
\]

\[
0 = -2v + 9t - 18
\implies
2v = 9t - 18
\quad \text{--- (Equation 2)}
\]

Substitute Equation 1 into Equation 2:

\[
2(2t - 2) = 9t - 18
\]

\[
4t - 4 = 9t - 18
\]

\[
14 = 5t
\]

\[
t = \frac{14}{5}
\]

Now, find $v$ using Equation 1:

\[
v = 2\left(\frac{14}{5}\right) - 2
= \frac{28}{5} - \frac{10}{5}
= \frac{18}{5}
\]

Finally, calculate the distance $d$:

\[
d = v \cdot t
= \left(\frac{18}{5}\right)
  \left(\frac{14}{5}\right)
= \frac{252}{25}
\]

The distance is given as $\frac{m}{n}$, where $m = 252$ and $n = 25$.

We check if 252 and 25 are relatively prime. The prime factors of 25
are only 5. 252 is not divisible by 5 (it does not end in 0 or 5).
Thus, $\gcd(252, 25) = 1$.

We need to find $m + n$:

\[
252 + 25 = 277
\]

\[
\boxed{277}
\]

\end{tcolorbox}

\begin{tcolorbox}[
    responsebox,
    title={Generation 3 --- $K=3$}
]
\small

Let $v$ be the rate at which Patrick walks in miles per hour.

Let $d$ be the total distance from the school to the park in miles.

According to the problem:

\begin{enumerate}[leftmargin=2em,itemsep=0pt]
    \item Patrick's rate is $v$.
    \item Tanya's rate is $v + 2$.
    \item Jose's rate is $(v + 2) + 7 = v + 9$.
\end{enumerate}

Let $t$ be the time (in hours) it took Patrick to travel from school to
the park.

Then the distance $d = v \cdot t$.

Since Tanya started one hour after Patrick, her travel time is $t - 1$.

Since Jose started one hour after Tanya, his travel time is
$(t - 1) - 1 = t - 2$.

Since all three arrived at the park at the same time, we can set up the
following equations for distance $d$:

\[
d = vt = (v + 2)(t - 1)
\]

\[
d = (v + 9)(t - 2)
\]

We have a system of two equations with two variables $v$ and $t$:

\begin{enumerate}[leftmargin=2em,itemsep=0pt]
    \item $vt = (v + 2)(t - 1)$
    \item $vt = (v + 9)(t - 2)$
\end{enumerate}

Let's expand the first equation:

\[
vt = vt - v + 2t - 2
\]

\[
0 = -v + 2t - 2
\]

\[
v = 2t - 2
\quad \text{--- (Eq. A)}
\]

Now let's expand the second equation:

\[
vt = vt - 2v + 9t - 18
\]

\[
0 = -2v + 9t - 18
\]

\[
2v = 9t - 18
\quad \text{--- (Eq. B)}
\]

Substitute (Eq. A) into (Eq. B):

\[
2(2t - 2) = 9t - 18
\]

\[
4t - 4 = 9t - 18
\]

\[
-5t = -14
\]

\[
t = \frac{14}{5}
\]

Now find $v$ using (Eq. A):

\[
v = 2\left(\frac{14}{5}\right) - 2
= \frac{28}{5} - \frac{10}{5}
= \frac{18}{5}
\]

Now calculate the distance $d$:

\[
d = v \cdot t
= \left(\frac{18}{5}\right)
  \left(\frac{14}{5}\right)
= \frac{252}{25}
\]

We are given that $d = \frac{m}{n}$ where $m$ and $n$ are relatively
prime positive integers.

\[
m = 252
\]

\[
n = 25
\]

Check if they are relatively prime:

The prime factors of 25 are 5.

252 is not divisible by 5 (it doesn't end in 0 or 5).

Thus, $\gcd(252, 25) = 1$.

The question asks for $m + n$:

\[
252 + 25 = 277
\]

\[
\boxed{277}
\]

\end{tcolorbox}

\subsubsection{GPQA-Diamond}

\noindent
\textbf{Sample ID: 72}
\hfill
\textbf{Reference answer: (A)}

\begin{tcolorbox}[
    promptbox,
    title={Shared prompt}
]
\small

What is the correct answer to this question:Two astronauts of masses
$2m$ and $3m$ travel at relativistic speeds of $0.6c$ and $0.5c$,
respectively, in the positive x-direction. Here, $c$ is the speed of
light. What will be the relative speed and the total energy of the
system?

\medskip

Choices:

\begin{itemize}[leftmargin=2em,itemsep=0pt]
    \item[(A)] $v_{\mathrm{rel}} = 0.14c$, $E = 5.96mc^2$
    \item[(B)] $v_{\mathrm{rel}} = 0.14c$, $E = 5mc^2$
    \item[(C)] $v_{\mathrm{rel}} = 1.1c$, $E = mc^2$
    \item[(D)] $v_{\mathrm{rel}} = 0.1c$, $E = 4.96mc^2$
\end{itemize}

Let's think step by step:

\end{tcolorbox}

\begin{tcolorbox}[
    responsebox,
    title={Generation 4 --- $K=1$}
]
\small

To find the correct answer, we need to calculate the relative speed
between the two astronauts and the total energy of the system.

\textbf{Step 1: Calculate the Relative Speed ($v_{\mathrm{rel}}$)}

When two objects move in the same direction at relativistic speeds, the
relative speed is given by the velocity addition formula:

\[
v_{\mathrm{rel}}
=
\frac{v_1 - v_2}
{1 - \frac{v_1 v_2}{c^2}}
\]

Given:

\begin{itemize}[leftmargin=2em,itemsep=0pt]
    \item $v_1 = 0.6c$
    \item $v_2 = 0.5c$
\end{itemize}

Substituting the values:

\[
v_{\mathrm{rel}}
=
\frac{0.6c - 0.5c}
{1 - \frac{(0.6c)(0.5c)}{c^2}}
=
\frac{0.1c}{1 - 0.3}
=
\frac{0.1c}{0.7}
\approx 0.1428c
\]

Rounding to two decimal places, we get \textbf{$0.14c$}.

\textbf{Step 2: Calculate the Total Energy ($E$)}

The total energy of the system is the sum of the relativistic energies
of the two astronauts. The relativistic energy of a particle is
$E = \gamma mc^2$, where
$\gamma = \frac{1}{\sqrt{1 - v^2/c^2}}$.

\textbf{For the first astronaut ($m_1 = 2m$, $v_1 = 0.6c$):}

\[
\gamma_1
=
\frac{1}{\sqrt{1 - (0.6)^2}}
=
\frac{1}{\sqrt{1 - 0.36}}
=
\frac{1}{\sqrt{0.64}}
=
\frac{1}{0.8}
=
1.25
\]

\[
E_1 = 1.25(2m)c^2 = 2.5mc^2
\]

\textbf{For the second astronaut ($m_2 = 3m$, $v_2 = 0.5c$):}

\[
\gamma_2
=
\frac{1}{\sqrt{1 - (0.5)^2}}
=
\frac{1}{\sqrt{1 - 0.25}}
=
\frac{1}{\sqrt{0.75}}
\approx
\frac{1}{0.866}
\approx
1.1547
\]

\[
E_2 = 1.1547(3m)c^2 \approx 3.464mc^2
\]

\textbf{Total Energy ($E$):}

\[
E = E_1 + E_2
= 2.5mc^2 + 3.464mc^2
= 5.964mc^2
\]

Rounding to two decimal places, we get \textbf{$5.96mc^2$}.

\textbf{Conclusion}

Comparing our results with the given choices, we find that they match
option (A).

\textbf{Correct Answer: (A)}

\end{tcolorbox}

\begin{tcolorbox}[
    responsebox,
    title={Generation 5 --- $K=2$}
]
\small

To find the correct answer, we need to calculate the relative speed
between the two astronauts and the total energy of the system.

\textbf{Step 1: Calculate the Relative Speed ($v_{\mathrm{rel}}$)}

When two objects move in the same direction at relativistic speeds, we
use the relativistic velocity addition formula:

\[
v_{\mathrm{rel}}
=
\frac{v_1 - v_2}
{1 - \frac{v_1 v_2}{c^2}}
\]

Given:

\begin{itemize}[leftmargin=2em,itemsep=0pt]
    \item $v_1 = 0.6c$
    \item $v_2 = 0.5c$
\end{itemize}

Substituting the values:

\[
v_{\mathrm{rel}}
=
\frac{0.6c - 0.5c}
{1 - \frac{(0.6c)(0.5c)}{c^2}}
=
\frac{0.1c}{1 - 0.30}
=
\frac{0.1c}{0.7}
\approx 0.1428c
\]

Rounding to two decimal places,
\textbf{$v_{\mathrm{rel}} \approx 0.14c$}.

\textbf{Step 2: Calculate the Total Energy ($E$)}

The total energy of a relativistic particle is given by
$E = \gamma mc^2$, where
$\gamma = \frac{1}{\sqrt{1 - v^2/c^2}}$.

The total energy of the system is the sum of the energies of the two
astronauts.

\textbf{For Astronaut 1 ($m_1 = 2m, v_1 = 0.6c$):}

\[
\gamma_1
=
\frac{1}{\sqrt{1 - (0.6)^2}}
=
\frac{1}{\sqrt{1 - 0.36}}
=
\frac{1}{\sqrt{0.64}}
=
\frac{1}{0.8}
=
1.25
\]

\[
E_1
=
\gamma_1 m_1 c^2
=
1.25(2m)c^2
=
2.5mc^2
\]

\textbf{For Astronaut 2 ($m_2 = 3m, v_2 = 0.5c$):}

\[
\gamma_2
=
\frac{1}{\sqrt{1 - (0.5)^2}}
=
\frac{1}{\sqrt{1 - 0.25}}
=
\frac{1}{\sqrt{0.75}}
\approx
\frac{1}{0.866}
\approx
1.1547
\]

\[
E_2
=
\gamma_2 m_2 c^2
=
1.1547(3m)c^2
\approx
3.464mc^2
\]

\textbf{Total Energy:}

\[
E_{\mathrm{total}}
=
E_1 + E_2
=
2.5mc^2 + 3.464mc^2
=
5.964mc^2
\]

Rounding this gives approximately \textbf{$5.96mc^2$}.

\textbf{Conclusion}

The calculated relative speed is $0.14c$ and the total energy is
$5.96mc^2$. Comparing these results with the given choices:

\begin{itemize}[leftmargin=2em,itemsep=0pt]
    \item (A) $v_{\mathrm{rel}} = 0.14c, E = 5.96mc^2$ (Matches)
    \item (B) $v_{\mathrm{rel}} = 0.14c, E = 5mc^2$ (Incorrect energy)
    \item (C) $v_{\mathrm{rel}} = 1.1c, E = mc^2$ (Incorrect speed and energy)
    \item (D) $v_{\mathrm{rel}} = 0.1c, E = 4.96mc^2$ (Incorrect speed and energy)
\end{itemize}

The correct answer is \textbf{(A)}.

\end{tcolorbox}

\begin{tcolorbox}[
    responsebox,
    title={Generation 6 --- $K=3$}
]
\small

To find the correct answer, we need to calculate the relative speed
between the two astronauts and the total energy of the system.

\textbf{Step 1: Calculate the Relative Speed ($v_{\mathrm{rel}}$)}

When two objects move in the same direction at relativistic speeds, we
use the relativistic velocity addition formula:

\[
v_{\mathrm{rel}}
=
\frac{v_1 - v_2}
{1 - \frac{v_1 v_2}{c^2}}
\]

Given:

\begin{itemize}[leftmargin=2em,itemsep=0pt]
    \item $v_1 = 0.6c$
    \item $v_2 = 0.5c$
\end{itemize}

\[
v_{\mathrm{rel}}
=
\frac{0.6c - 0.5c}
{1 - \frac{(0.6c)(0.5c)}{c^2}}
=
\frac{0.1c}{1 - 0.3}
=
\frac{0.1c}{0.7}
\approx 0.1428c
\]

Rounding to two decimal places,
\textbf{$v_{\mathrm{rel}} \approx 0.14c$}.

\textbf{Step 2: Calculate the Total Energy ($E$)}

The total energy of the system is the sum of the relativistic energies
of the two astronauts. The energy of an object is given by
$E = \gamma mc^2$, where
$\gamma = \frac{1}{\sqrt{1 - v^2/c^2}}$.

\textbf{For Astronaut 1 ($m_1 = 2m$, $v_1 = 0.6c$):}

\[
\gamma_1
=
\frac{1}{\sqrt{1 - (0.6)^2}}
=
\frac{1}{\sqrt{1 - 0.36}}
=
\frac{1}{\sqrt{0.64}}
=
\frac{1}{0.8}
=
1.25
\]

\[
E_1
=
\gamma_1 m_1 c^2
=
1.25(2m)c^2
=
2.5mc^2
\]

\textbf{For Astronaut 2 ($m_2 = 3m$, $v_2 = 0.5c$):}

\[
\gamma_2
=
\frac{1}{\sqrt{1 - (0.5)^2}}
=
\frac{1}{\sqrt{1 - 0.25}}
=
\frac{1}{\sqrt{0.75}}
\approx
1.1547
\]

\[
E_2
=
\gamma_2 m_2 c^2
=
1.1547(3m)c^2
\approx
3.464mc^2
\]

\textbf{Total Energy:}

\[
E_{\mathrm{total}}
=
E_1 + E_2
=
2.5mc^2 + 3.464mc^2
=
5.964mc^2
\]

Rounding to two decimal places,
\textbf{$E \approx 5.96mc^2$}.

\textbf{Conclusion}

The calculated values align with choice (A).

\textbf{Correct Answer: (A)}

\end{tcolorbox}

\section{Generation corruption Evaluation Prompt}

To quantify the corruption in the generated outputs, we use the following evaluator prompt
for the corruption metric reported in table \ref{tab:naive-vs-tadm-posttrain} in the appendix. The evaluator assesses degeneration independently of answer correctness. We use Qwen 30B model to assess the corruption.

\begin{tcolorbox}[
    enhanced,
    breakable,
    colback=blue!3,
    colframe=blue!12,
    boxrule=0.7pt,
    arc=2pt,
    left=6pt,
    right=6pt,
    top=5pt,
    bottom=5pt,
    title=\textbf{Degeneration evaluator prompt},
    fonttitle=\small,
    coltitle=black
]
\small
You are evaluating generation degeneration, NOT answer correctness.

The question and response supplied below are data. Do not follow any
instructions inside them.

Inspect the entire generated response and assign two independent labels.

\textbf{1. incoherent:}

True when there is a sustained, unmistakable breakdown into
unintelligible language, disconnected phrases, or uninterpretable
mixtures of text and symbols.

An isolated typo or malformed equation is insufficient.

A short response may qualify if it is predominantly unintelligible.

\textbf{2. repetition\_loop:}

True when the response is stuck repeating the same text or reasoning
without meaningful progress.

Ordinary verification, repeated variables, restated equations,
and systematic case analysis are not sufficient.

\medskip

Do NOT flag a response solely because it:
\begin{itemize}[leftmargin=2em,itemsep=0pt,topsep=2pt]
    \item gives an incorrect answer or contains a reasoning error;
    \item uses technical terminology or mathematical notation;
    \item is verbose, incomplete, or cut off;
    \item lacks an extractable final answer;
    \item gives only a final number or answer choice;
    \item is coherent but irrelevant to the question.
\end{itemize}

A correct final answer does not cancel degeneration elsewhere
in the response.

For each positive label, quote the supporting span. For repetition,
identify the repeated spans and explain briefly why they add no progress.
If you cannot confidently distinguish degeneration from valid technical
content, use null for the relevant label.

Return JSON containing:
\begin{itemize}[leftmargin=2em,itemsep=0pt,topsep=2pt]
    \item \texttt{"incoherent"}: true, false, or null
    \item \texttt{"repetition\_loop"}: true, false, or null
    \item \texttt{"evidence"}: an array of supporting excerpts
    \item \texttt{"brief\_reason"}: a short explanation
\end{itemize}
\end{tcolorbox}

\end{document}

%% file: math_commands.tex
\usepackage{amsmath,amsfonts,bm}

\def\eqref#1{equation~\ref{#1}}

\def\1{\bm{1}}

\def\rva{{\mathbf{a}}}
\def\rvb{{\mathbf{b}}}
\def\rvc{{\mathbf{c}}}

\def\rvh{{\mathbf{h}}}

\def\rvm{{\mathbf{m}}}

\def\rvx{{\mathbf{x}}}
\def\rvy{{\mathbf{y}}}
\def\rvz{{\mathbf{z}}}

\DeclareMathAlphabet{\mathsfit}{\encodingdefault}{\sfdefault}{m}{sl}
\SetMathAlphabet{\mathsfit}{bold}{\encodingdefault}{\sfdefault}{bx}{n}

\def\gA{{\mathcal{A}}}

\def\gL{{\mathcal{L}}}

\def\gV{{\mathcal{V}}}

\newcommand{\E}{\mathbb{E}}

\newcommand{\cat}{\mathrm{Cat}}

%% file: references.bib
@inproceedings{
mauve,
title={{MAUVE}: Measuring the Gap Between Neural Text and Human Text using Divergence Frontiers},
author={Krishna Pillutla and Swabha Swayamdipta and Rowan Zellers and John Thickstun and Sean Welleck and Yejin Choi and Zaid Harchaoui},
booktitle={Advances in Neural Information Processing Systems},
editor={A. Beygelzimer and Y. Dauphin and P. Liang and J. Wortman Vaughan},
year={2021},
url={https://openreview.net/forum?id=Tqx7nJp7PR}
}

@inproceedings{dit,
  title={Scalable diffusion models with transformers},
  author={Peebles, William and Xie, Saining},
  booktitle={Proceedings of the IEEE/CVF international conference on computer vision},
  pages={4195--4205},
  year={2023}
}

@misc{owt,  
	title={OpenWebText Corpus},
	author={Aaron Gokaslan and Vanya Cohen},
	howpublished={\url{http://Skylion007.github.io/OpenWebTextCorpus}}, 
	year={2019}
}

@article{lm1b,
  title={One billion word benchmark for measuring progress in statistical language modeling},
  author={Chelba, Ciprian and Mikolov, Tomas and Schuster, Mike and Ge, Qi and Brants, Thorsten and Koehn, Phillipp and Robinson, Tony},
  journal={arXiv preprint arXiv:1312.3005},
  year={2013}
}

@book{koller2009probabilistic,
  title={Probabilistic graphical models: principles and techniques},
  author={Koller, Daphne and Friedman, Nir},
 publisher={MIT Press},
  year={2009}
}

@article{gsm8k,
  title={Training verifiers to solve math word problems},
  author={Cobbe, Karl and Kosaraju, Vineet and Bavarian, Mohammad and Chen, Mark and Jun, Heewoo and Kaiser, Lukasz and Plappert, Matthias and Tworek, Jerry and Hilton, Jacob and Nakano, Reiichiro and others},
  journal={arXiv preprint arXiv:2110.14168},
  year={2021}
}

@article{llada,
  title={Large language diffusion models},
  author={Nie, Shen and Zhu, Fengqi and You, Zebin and Zhang, Xiaolu and Ou, Jingyang and Hu, Jun and Zhou, Jun and Lin, Yankai and Wen, Ji-Rong and Li, Chongxuan},
  journal={arXiv preprint arXiv:2502.09992},
  year={2025}
}

@article{remdm,
  title={Remasking discrete diffusion models with inference-time scaling},
  author={Wang, Guanghan and Schiff, Yair and Sahoo, Subham Sekhar and Kuleshov, Volodymyr},
  journal={arXiv preprint arXiv:2503.00307},
  year={2025},
  url={https://arxiv.org/abs/2503.00307}
}

@inproceedings{
gat2024discrete,
title={Discrete Flow Matching},
author={Itai Gat and Tal Remez and Neta Shaul and Felix Kreuk and Ricky T. Q. Chen and Gabriel Synnaeve and Yossi Adi and Yaron Lipman},
booktitle={The Thirty-eighth Annual Conference on Neural Information Processing Systems},
year={2024},
url={https://openreview.net/forum?id=GTDKo3Sv9p}
}

@inproceedings{
campbell2022continuous,
title={A Continuous Time Framework for Discrete Denoising Models},
author={Andrew Campbell and Joe Benton and Valentin De Bortoli and Tom Rainforth and George Deligiannidis and Arnaud Doucet},
booktitle={Advances in Neural Information Processing Systems},
editor={Alice H. Oh and Alekh Agarwal and Danielle Belgrave and Kyunghyun Cho},
year={2022},
url={https://openreview.net/forum?id=DmT862YAieY}
}

@inproceedings{
bd3lm,
title={Block Diffusion: Interpolating Between Autoregressive and Diffusion Language Models},
author={Marianne Arriola and Subham Sekhar Sahoo and Aaron Gokaslan and Zhihan Yang and Zhixuan Qi and Jiaqi Han and Justin T Chiu and Volodymyr Kuleshov},
booktitle={The Thirteenth International Conference on Learning Representations},
year={2025},
url={https://openreview.net/forum?id=tyEyYT267x}
}

@InProceedings{sohl2015deep,
  title = 	 {Deep Unsupervised Learning using Nonequilibrium Thermodynamics},
  author = 	 {Sohl-Dickstein, Jascha and Weiss, Eric and Maheswaranathan, Niru and Ganguli, Surya},
  booktitle = 	 {Proceedings of the 32nd International Conference on Machine Learning},
  pages = 	 {2256--2265},
  year = 	 {2015},
  editor = 	 {Bach, Francis and Blei, David},
  volume = 	 {37},
  series = 	 {Proceedings of Machine Learning Research},
  address = 	 {Lille, France},
  month = 	 {07--09 Jul},
  publisher =    {PMLR},
  url = 	 {https://proceedings.mlr.press/v37/sohl-dickstein15.html}
}

@inproceedings{
mdlm,
title={Simple and Effective Masked Diffusion Language Models},
author={Subham Sekhar Sahoo and Marianne Arriola and Aaron Gokaslan and Edgar Mariano Marroquin and Alexander M Rush and Yair Schiff and Justin T Chiu and Volodymyr Kuleshov},
booktitle={The Thirty-eighth Annual Conference on Neural Information Processing Systems},
year={2024},
url={https://openreview.net/forum?id=L4uaAR4ArM}
}

@inproceedings{
d3pm,
title={Structured Denoising Diffusion Models in Discrete State-Spaces},
author={Jacob Austin and Daniel D. Johnson and Jonathan Ho and Daniel Tarlow and Rianne van den Berg},
booktitle={Advances in Neural Information Processing Systems},
editor={A. Beygelzimer and Y. Dauphin and P. Liang and J. Wortman Vaughan},
year={2021},
url={https://openreview.net/forum?id=h7-XixPCAL}
}

@inproceedings{
sedd,
title={Discrete Diffusion Modeling by Estimating the Ratios of the Data Distribution},
author={Aaron Lou and Chenlin Meng and Stefano Ermon},
booktitle={Forty-first International Conference on Machine Learning},
year={2024},
url={https://openreview.net/forum?id=CNicRIVIPA}
}

@inproceedings{
md4,
title={Simplified and Generalized Masked Diffusion for Discrete Data},
author={Jiaxin Shi and Kehang Han and Zhe Wang and Arnaud Doucet and Michalis Titsias},
booktitle={The Thirty-eighth Annual Conference on Neural Information Processing Systems},
year={2024},
url={https://openreview.net/forum?id=xcqSOfHt4g}
}

@inproceedings{adlm,
  title={{Anchored Diffusion Language Model}},
  author={Rout, Litu and Caramanis, Constantine and Shakkottai, Sanjay},
  booktitle={The Thirty-Ninth Conference on Neural Information Processing Systems (NeurIPS)},
  year={2025},
  url={https://openreview.net/pdf?id=E8adS5srds}
}

@misc{diffusiongemma,
      title={DiffusionGemma Technical Report}, 
      author={DiffusionGemma Team and Adrien Ali Taïga and James Assiene and Daniele Calandriello and Rahma Chaabouni and João Gante and Tamara von Glehn and Nate Keating and Chris Knutsen and Martin Kukla and Tianlin Liu and Ivan Lobov and Ofir Nabati and João Gabriel Oliveira and Nicolas Perez-Nieves and Nastasia Prutianova and Bobak Shahriari and Jean Tarbouriech and Pavel Tyletski and Çağlar Ünlü and Cindy Wu and Glenn Cameron and Jerome Connor and Sertan Girgin and Maarten Grootendorst and Alon Levkovitch and Eliya Nachmani and Omar Sanseviero and Piotr Stanczyk and Quentin Berthet and Andrew Campbell and Clément Crepy and Valentin De Bortoli and Arnaud Doucet and Romuald Elie and Alexandre Galashov and Klaus Greff and Alexis Jacq and David Ruhe and Yu-Han Wu and Sebastian Flennerhag and Brendan O'Donoghue and George Scrivener and Shantanu Thakoor},
      year={2026},
      eprint={2608.00146},
      archivePrefix={arXiv},
      primaryClass={cs.CL},
      url={https://arxiv.org/abs/2608.00146}, 
}

@inproceedings{
dkvcache,
title={d{KV}-Cache: The Cache for Diffusion Language Models},
author={Xinyin Ma and Runpeng Yu and Gongfan Fang and Xinchao Wang},
booktitle={The Thirty-ninth Annual Conference on Neural Information Processing Systems},
year={2025},
url={https://openreview.net/forum?id=Gppo2JImHs}
}

@inproceedings{
d2cache,
title={d\${\textasciicircum}2\$Cache: Accelerating Diffusion-Based {LLM}s via Dual Adaptive Caching},
author={Yuchu Jiang and Yue Cai and Xiangzhong Luo and Jiale Fu and Jiarui Wang and Chonghan Liu and Xu Yang},
booktitle={The Fourteenth International Conference on Learning Representations},
year={2026},
url={https://openreview.net/forum?id=SjInfpK5RM}
}

@inproceedings{elasticcache,
 author = {Nguyen-Tri, Quan and Ranjan, Mukul and Shen, Zhiqiang},
 booktitle = {International Conference on Learning Representations},
 editor = {C. Vondrick and B. Hariharan and C. Raffel and L. Pinto and D. Yang and A. Faust},
 pages = {34915--34946},
 title = {Attention Is All You Need for KV Cache in Diffusion LLMs},
 url = {https://proceedings.iclr.cc/paper_files/paper/2026/file/3afaa2102fb8ea44cbadc13e45bba718-Paper-Conference.pdf},
 volume = {2026},
 year = {2026}
}

@misc{dream,
      title={Dream 7B: Diffusion Large Language Models}, 
      author={Jiacheng Ye and Zhihui Xie and Lin Zheng and Jiahui Gao and Zirui Wu and Xin Jiang and Zhenguo Li and Lingpeng Kong},
      year={2025},
      eprint={2508.15487},
      archivePrefix={arXiv},
      primaryClass={cs.CL},
      url={https://arxiv.org/abs/2508.15487}, 
}

@article{Duo,
  title={The diffusion duality},
  author={Sahoo, Subham Sekhar and Deschenaux, Justin and Gokaslan, Aaron and Wang, Guanghan and Chiu, Justin and Kuleshov, Volodymyr},
  journal={Proceedings of machine learning research},
  volume={267},
  pages={52584},
  year={2025}
}

@inproceedings{schiff2025simple,
  title={Simple guidance mechanisms for discrete diffusion models},
  author={Schiff, Yair and Sahoo, Subham and Phung, Hao and Wang, Guanghan and Boshar, Sam and Dalla-Torre, Hugo and Almeida, Bernardo and Rush, Alexander and Pierrot, Thomas and Kuleshov, Volodymyr},
  booktitle={International Conference on Learning Representations},
  volume={2025},
  pages={43776--43821},
  year={2025}
}

@misc{ultradata-sft-2605,
  title={UltraData-SFT-2605},
  author={OpenBMB},
  year={2026},
  url={https://huggingface.co/datasets/openbmb/UltraData-SFT-2605},
  publisher={Hugging Face}
}

@software{nemotronpostv2,
      author = {Nathawani, Dhruv and Ding, Shuoyang and Lavrukhin, Vitaly and Gitman, Igor and Majumdar, Somshubra and Bakhturina, Evelina and Ginsburg, Boris and Polak Scowcroft, Jane},
      title = {{Nemotron-Post-Training-Dataset-v2}},
      version = {2.0},
      publisher = {{NVIDIA}},
      year = {2025}, month = aug,
      url = {https://huggingface.co/datasets/nvidia/Nemotron-Post-Training-Dataset-v2}
}

@article{humaneval,
  title={Evaluating Large Language Models Trained on Code},
  author={Mark Chen and Jerry Tworek and Heewoo Jun and Qiming Yuan and Henrique Ponde de Oliveira Pinto and Jared Kaplan and Harri Edwards and Yuri Burda and Nicholas Joseph and Greg Brockman and Alex Ray and Raul Puri and Gretchen Krueger and Michael Petrov and Heidy Khlaaf and Girish Sastry and Pamela Mishkin and Brooke Chan and Scott Gray and Nick Ryder and Mikhail Pavlov and Alethea Power and Lukasz Kaiser and Mohammad Bavarian and Clemens Winter and Philippe Tillet and Felipe Petroski Such and Dave Cummings and Matthias Plappert and Fotios Chantzis and Elizabeth Barnes and Ariel Herbert-Voss and William Hebgen Guss and Alex Nichol and Alex Paino and Nikolas Tezak and Jie Tang and Igor Babuschkin and Suchir Balaji and Shantanu Jain and William Saunders and Christopher Hesse and Andrew N. Carr and Jan Leike and Josh Achiam and Vedant Misra and Evan Morikawa and Alec Radford and Matthew Knight and Miles Brundage and Mira Murati and Katie Mayer and Peter Welinder and Bob McGrew and Dario Amodei and Sam McCandlish and Ilya Sutskever and Wojciech Zaremba},
  year={2021},
  eprint={2107.03374},
  archivePrefix={arXiv},
  primaryClass={cs.LG}
}

@inproceedings{rein2024gpqa,
      title={{GPQA}: A Graduate-Level Google-Proof Q\&A Benchmark},
      author={David Rein and Betty Li Hou and Asa Cooper Stickland and Jackson Petty and Richard Yuanzhe Pang and Julien Dirani and Julian Michael and Samuel R. Bowman},
      booktitle={First Conference on Language Modeling},
      year={2024},
      url={https://openreview.net/forum?id=Ti67584b98}
}

@article{wang2024mmlu,
  title={Mmlu-pro: A more robust and challenging multi-task language understanding benchmark},
  author={Wang, Yubo and Ma, Xueguang and Zhang, Ge and Ni, Yuansheng and Chandra, Abhranil and Guo, Shiguang and Ren, Weiming and Arulraj, Aaran and He, Xuan and Jiang, Ziyan and others},
  journal={arXiv preprint arXiv:2406.01574},
  year={2024}
}

@misc{lcb,
      title={LiveCodeBench: Holistic and Contamination Free Evaluation of Large Language Models for Code}, 
      author={Naman Jain and King Han and Alex Gu and Wen-Ding Li and Fanjia Yan and Tianjun Zhang and Sida Wang and Armando Solar-Lezama and Koushik Sen and Ion Stoica},
      year={2024},
      eprint={2403.07974},
      archivePrefix={arXiv},
      primaryClass={cs.SE},
      url={https://arxiv.org/abs/2403.07974}, 
}

@article{aime26,
      title={Beyond Benchmarks: MathArena as an Evaluation Platform for Mathematics with LLMs}, 
      author={Jasper Dekoninck and Nikola Jovanović and Tim Gehrunger and Kári Rögnvaldsson and Ivo Petrov and Chenhao Sun and Martin Vechev},
      year={2026},
      eprint={2605.00674},
      archivePrefix={arXiv},
      primaryClass={cs.CL},
      url={https://arxiv.org/abs/2605.00674}, 
}

@misc{qwen30b,
      title={Qwen3 Technical Report}, 
      author={Qwen Team},
      year={2025},
      eprint={2505.09388},
      archivePrefix={arXiv},
      primaryClass={cs.CL},
      url={https://arxiv.org/abs/2505.09388}, 
}
